%% file: camera.tex
\documentclass{article}

\usepackage[main,final]{neurips_2025}
\usepackage[utf8]{inputenc} 
\usepackage[T1]{fontenc}    
\usepackage{hyperref}       
\usepackage{url}            
\usepackage{booktabs}       
\usepackage{amsfonts}       
\usepackage{nicefrac}       
\usepackage{microtype}      
\usepackage{xcolor}         

\usepackage{subcaption}
\usepackage{graphicx}
\usepackage{wrapfig}
\usepackage{amsmath}
\usepackage{amssymb}
\usepackage{algorithm,algorithmicx,algpseudocode}
\usepackage{rotating}
\usepackage{multirow} 
\usepackage{pifont}
\usepackage{dsfont}
\usepackage{colortbl}
\usepackage{tikz}
\newtheorem{theorem}{Theorem}
\newcommand{\tikzxmark}{%
\tikz[scale=0.23] {
    \draw[line width=1,line cap=round] (0,0) to [bend left=6] (1,1);
    \draw[line width=1,line cap=round] (0.2,0.95) to [bend right=3] (0.8,0.05);
}}
\newcommand{\tikzcmark}{%
\tikz[scale=0.23] {
    \draw[line width=1,line cap=round] (0.25,0) to [bend left=10] (1,1);
    \draw[line width=1,line cap=round] (0,0.35) to [bend right=1] (0.23,0);
}}

\input{defs}

\newcommand\unimodal[1]{\cellcolor{gray!25}#1}

\newcommand{\first}[1]{\bf{#1}}
\newcommand{\second}[1]{\underline{#1}}

\newcommand{\std}[1]{{\scriptsize$\pm$#1}}
\usepackage{makecell}

\def\CREMAD{{CREMAD}}
\def\Kinetics{{KSounds}}
\def\Twitter{{Twitter}}
\def\Sarcasm{{Sarcasm}}
\def\NVGesture{{NVGesture}}
\def\VGGSound{{VGGSound}}

\newtheorem{assumption}{Assumption}
\newtheorem{lemma}{Lemma}
\newtheorem{proof}{Proof}

\newtheorem{thm}{Theorem}

\title{Rethinking Multimodal Learning from the Perspective of Mitigating Classification Ability Disproportion}

\author{
  Qing-Yuan Jiang$^\dag$$^\ddag$, Longfei Huang$^\dag$, Yang Yang$^{\dag}$\thanks{Corresponding author.}\\
  $^\dag$Nanjing University of Science and Technology\\
  $^\ddag$State Key Lab. for Novel Software Technology, Nanjing University, P.R. China\\
  \texttt{\{jiangqy, hlf, yyang\}@njust.edu.cn} \\
}

\begin{document}

\maketitle

\begin{abstract}
Multimodal learning~(MML) is significantly constrained by modality imbalance, leading to suboptimal performance in practice. While existing approaches primarily focus on balancing the learning of different modalities to address this issue, they fundamentally overlook the inherent disproportion in model classification ability, which serves as the primary cause of this phenomenon. In this paper, we propose a novel multimodal learning approach to dynamically balance the classification ability of weak and strong modalities by incorporating the principle of boosting. Concretely, we first propose a sustained boosting algorithm in multimodal learning by simultaneously optimizing the classification and residual errors. Subsequently, we introduce an adaptive classifier assignment strategy to dynamically facilitate the classification performance of the weak modality. Furthermore, we theoretically analyze the convergence property of the cross-modal gap function, ensuring the effectiveness of the proposed boosting scheme. To this end, the classification ability of strong and weak modalities is expected to be balanced, thereby mitigating the imbalance issue. Empirical experiments on widely used datasets reveal the superiority of our method through comparison with various state-of-the-art~(SOTA) multimodal learning baselines. The source code is available at \url{https://github.com/njustkmg/NeurIPS25-AUG}.
\end{abstract}

\section{Introduction}
In recent years, multimodal learning~\cite{MMDL:conf/icml/NgiamKKNLN11,MML:conf/ijcai/YangWZX019,SMV:conf/cvpr/YakeRZD24,LFM:conf/nips/0074WJ024,ReconBoost:conf/icml/CongHua24} has received growing attention for its ability to effectively integrate heterogeneous information. As extra information from multimodal data can be utilized, multimodal learning is expected to achieve better performance compared with unimodal approaches. However, contrary to expectations, multimodal learning has been surprisingly shown to underperform compared to unimodal ones in certain scenarios~\cite{OGR-GB:conf/cvpr/WangTF20,OGM:conf/cvpr/PengWD0H22,AMSS:journals/pami/YangPJXT25}.

The root of this problem lies in the existence of the modality imbalance~\cite{OGR-GB:conf/cvpr/WangTF20}. Concretely, different modalities in a joint-training paradigm typically converge at different speeds~\cite{OGM:conf/cvpr/PengWD0H22,nvGesture:conf/icml/WuJCG22}. The faster-converging modality, i.e., strong modality~\cite{ARM:conf/ijcai/YangYZJ15}, tends to achieve higher performance, while the weak modality performs poorly. Subsequently, this disproportion in classification ability often leads to modality imbalance~\cite{OGR-GB:conf/cvpr/WangTF20}, ultimately resulting in lower performance.

Researchers have explored the modality imbalance issue from various perspectives in multimodal learning~\cite{OGR-GB:conf/cvpr/WangTF20,OGM:conf/cvpr/PengWD0H22,MLA:conf/cvpr/ZhangYBY24}. Given the inconsistent learning progress between strong and weak modalities, a natural idea~\cite{OGR-GB:conf/cvpr/WangTF20,OGM:conf/cvpr/PengWD0H22,AGM:conf/iccv/LiLHLLZ23,MSLR:conf/acl/YaoM22}  is to manually intervene in their learning processes to achieve rebalancing. Another type of method is to bridge the information gap between modality training phases and enhance the interaction between different modalities during training. To be specific, impressive works~\cite{MLA:conf/cvpr/ZhangYBY24,DI-MML:conf/mm/FanXWLG24} such as MLA~\cite{MLA:conf/cvpr/ZhangYBY24}, ReconBoost~\cite{ReconBoost:conf/icml/CongHua24} and DI-MML~\cite{DI-MML:conf/mm/FanXWLG24} focus on bridging the learning gap of different modalities through injecting the optimization information between modalities. 

Although the above methods can rebalance multimodal learning, they focus more on balancing the learning process while failing to enhance the classification ability explicitly. Compared to weaker modalities, stronger modalities typically yield more robust classifiers due to their more sufficient information~\cite{ARM:conf/ijcai/YangYZJ15}. Is there a way to directly improve the performance of weak classifiers to balance the classification performance between strong and weak modalities? A natural choice is boosting~\cite{Adaboost:conf/eurocolt/FreundS95,GradientBoosting:journal/AS/FriedmanJ}, which utilizes the ensemble technique to enhance the ability of the weak classifier. We conduct a toy experiment to illustrate this idea on \CREMAD~dataset~\cite{CREMAD:journals/taffco/CaoCKGNV14}, where the classifier of weak modality is enhanced by the gradient boosting~\cite{GradientBoosting:journal/AS/FriedmanJ}. The results in Figure~\ref{fig:intro} present the comparison among naive MML, a model learning adjustment-based MML approach~(G-Blend~\cite{OGR-GB:conf/cvpr/WangTF20}), and gradient boosting-based MML~(MML w/ GB). For MML w/ GB, we apply the gradient boosting algorithm to further improve the trained video model using naive MML, while keeping the audio model fixed. We can find that the classification gap between video and audio modalities of naive MML and G-Blend is relatively large. More importantly, for MML w/ GB, the accuracy of audio modality remains unchanged, but the accuracy of video is greatly improved, leading to the improvement of overall accuracy. This demonstrates the feasibility and effectiveness of using boosting to balance the classification ability of strong and weak modalities in mitigating modality imbalance.

\begin{wrapfigure}{r}{0.5\textwidth}
\centering
\includegraphics[scale=0.85]{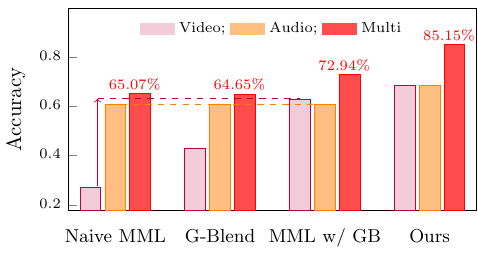}
\caption{Comparison with naive MML, gradient boosting based MML~(MML w/ GB), G-Blend~\cite{OGR-GB:conf/cvpr/WangTF20}, and Ours on \CREMAD~dataset. We find that enhancing the classification performance of the weak modality narrows the performance gap between the two modalities and improves overall performance.}
\label{fig:intro}
\end{wrapfigure}
According to the aforementioned observations, in this paper, we propose a novel multimodal learning approach by designing a sustained boosting algorithm to facilitate the classification ability of the weaker modality. Concretely, we first propose a sustained boosting algorithm by jointly optimizing the classification loss and residual error, aiming to enhance the classification performance of the weak modality. The algorithm takes features extracted by the encoder as input and provides predictions for the given data. Then, we employ the confident score~\cite{OGM:conf/cvpr/PengWD0H22} to monitor the learning status during joint training, and further propose an adaptive classifier assignment~(ACA) strategy to adjust the classifier of weak modality. To this end, we can enhance the classification ability for the weak modality, thereby rebalancing the classification ability of strong and weak modalities. Meanwhile, we theoretically show that, under the boosting algorithm framework, the gap between different losses is provably convergent. In Figure~\ref{fig:intro}, we present the classification enhancement results of our method~(Ours). We can find that the performance of our method outperforms that of MML w/ GB thanks to the sustained boosting and adaptive classifier assignment strategy. Furthermore, it is worth mentioning that ReconBoost~\cite{ReconBoost:conf/icml/CongHua24} also employs the gradient boosting algorithm for multimodal learning. However, unlike our approach, ReconBoost uses gradient boosting to iteratively learn complementary information across modalities. Our main contributions are outlined as follows: 
\begin{itemize}
\item A novel sustained boosting algorithm in MML is proposed. This algorithm aims to simultaneously minimize the classification and residual errors to facilitate the classification ability of the weak modality.
\item A novel adaptive classifier assignment strategy is proposed to dynamically enhance the classification ability of weak modality based on the learning status, thus rebalancing the classification ability of all modalities. 
\item We theoretically analyze the impact of the boosting algorithm on the loss gap between different modalities and prove its convergence. 
\item Experiments reveal that our approach can outperform SOTA baselines to achieve the best performance by a large margin on widely used datasets.
\end{itemize}

\section{Related Work}
\subsection{Rebalanced Multimodal Learning}
The goal of multimodal learning~\cite{MML:conf/ijcai/YangWZX019,MMDL:conf/icml/JiaYXCPPLSLD21,MMDL:conf/icml/NgiamKKNLN11,MMDL:conf/cvpr/SinghHGCGRK22,MML:conf/nips/HuangDXCZH21} is to fuse the multimodal information from diverse sensors. Compared to unimodal methods, multimodal learning can mine data information from different perspectives, thus the performance of multimodal learning should be better~\cite{MML:conf/cvpr/LvCHDL21,MML:conf/nips/SimonyanZ14,MML:conf/cvpr/HuLL16,MML:conf/cvpr/GaoOGT20}. However, due to heterogeneity of multimodal data, multimodal learning often encounters imbalance problems~\cite{OGR-GB:conf/cvpr/WangTF20,ModalCompetition:conf/icml/HuangLZYH22} in practice, leading to performance degeneration of multimodal learning. 

Early pioneering works~\cite{OGR-GB:conf/cvpr/WangTF20,OGM:conf/cvpr/PengWD0H22,PMR:conf/cvpr/Fan0WW023,MML:conf/cvpr/Ge0G00AILZ23} focus more on adaptively adjusting the learning procedure for different modalities. Representative approaches in this category employ different learning strategies, e.g., gradient modulation~\cite{OGM:conf/cvpr/PengWD0H22,AGM:conf/iccv/LiLHLLZ23} and learning rate adjustment~\cite{MSLR:conf/acl/YaoM22}, to rebalance the learning of weak and strong modalities. Other approaches including MLA~\cite{MLA:conf/cvpr/ZhangYBY24}, ReconBoost~\cite{ReconBoost:conf/icml/CongHua24} and IGM~\cite{IGM:conf/ijcai/JiangCY25} take a different path, focusing on enhancing the interaction between modalities to address the modality imbalance problem. For example, MLA~\cite{MLA:conf/cvpr/ZhangYBY24} designs an alternating algorithm to train different modalities iteratively. During the training phase, the interaction is enhanced by transferring the learning information between different modalities. ReconBoost~\cite{ReconBoost:conf/icml/CongHua24} balances modality learning by leveraging gradient boosting to capture information from other modalities during interactive learning. IGM~\cite{IGM:conf/ijcai/JiangCY25} employs a flat gradient modification strategy to enhance the interactive multimodal learning.

The aforementioned methods focus on rebalancing the learning process for weak and strong modalities while failing to explicitly facilitate the classification ability of the weak modality. In this paper, we aim to address the modality imbalance issue from facilitating the classification ability of weak modality and rebalancing the classification ability of weak and strong modalities. 

\subsection{Boosting Method}

Boosting algorithm~\cite{Adaboost:conf/eurocolt/FreundS95,GradientBoosting:journal/AS/FriedmanJ,Boost:conf/cvpr/LiuHMT14,Boost:conf/icml/CortesMS14,Boost:conf/nips/ProkhorenkovaGV18} is one of the most important algorithms in ensemble learning. The core idea of boosting is to integrate multiple learners to create a strong learner. Adaboost~\cite{Adaboost:conf/eurocolt/FreundS95}, one of the earliest boosting algorithms, adjusts the weights of incorrectly classified data points, giving more attention to the harder-to-classify examples in each iteration. Gradient boosting~\cite{GradientBoosting:journal/AS/FriedmanJ}, on the other hand, builds models in a stage-wise fashion, minimizing a loss function through gradient descent. It iteratively refines the overall model by fitting the negative gradient of the loss function~\cite{GB:journals/finr/NatekinK13} with respect to the model’s predictions.

The key advantage of boosting lies in its ability to improve model accuracy without requiring complex individual models. Therefore, boosting becomes the natural choice for improving the performance of weak classifiers.

\section{Methodology}\label{sec:AGB}

\subsection{Multimodal Learning}
For simplicity, we use two modalities, i.e., audio and video, for illustration. It is worth mentioning that our method can be easily adapted to cases with more than two modalities.

Assume that we have $N$ data points, each of which has audio and video modalities. Without loss of generality, we use $\X=\{(\x_i^a,\x^v_i)\}_{i=1}^N$ to denote the multimodal data, where $\x_i^a$ and $\x_i^v$ denote the $i$-th data point of audio and video, respectively. In addition, we are also given a category labels set $\Y=\{\;\y_i~\vert~\y_i\in\{0,1\}^K\}_{i=1}^N$, where $K$ denotes the number of category labels. Given the above training information $\X$ and $\Y$, the goal of multimodal learning is to train a model to fuse the multimodal information and predict its category label as accurately as possible.

For the sake of simplicity, we use superscript $o$ to indicate the module corresponding to a specific modality in this section, where $o\in\{a,v\}$. With the rapid growth of deep learning, representative MML approaches~\cite{MMDL:conf/icml/NgiamKKNLN11,OGR-GB:conf/cvpr/WangTF20,PMR:conf/cvpr/Fan0WW023,AGM:conf/iccv/LiLHLLZ23} have adopted deep neural network~(DNN) for multimodal learning. Following these methods, we also utilize DNN to construct our models. Specifically, we use $\phi^o(\cdot)$ to denote encoders. Then the features can be calculated by $\u^o=\phi^o(\x^o;\theta^o)$, where $\theta^o$ denotes the encoder parameters. Then, the prediction of given data can be calculated by a classifier $\psi^o(\cdot)$: $\p^o=\psi^o(\u^o;\Theta^o),$
where $\Theta^o$ denotes the parameters of the classifier. Based on $\p^o$ and its ground-truth, the objective function can be written as:
\begin{align}
\LM_{\mathrm{CE}}(\X^o,\Y)=\frac{1}{N}\sum_{i=1}^{N}\ell(\p^o_i,\y_i)=-\frac{1}{N}\sum_{i=1}^{N}\y_i^\top\log(\p^o_{i}),\label{obj:mml}
\end{align}
where $\Phi^o \triangleq \{\theta^o,\Theta^o\}$ denotes the parameters to be learned, $\X^o\triangleq \{\x^o_i\}_{i=1}^N $ and $\ell(\cdot)$ denotes the cross entropy loss.
\subsection{Sustained Boosting}
By training the model for each modality based on objective function~(\ref{obj:mml}), we can obtain multiple individual classifiers. Due to the existence of strong and weak modalities~\cite{ARM:conf/ijcai/YangYZJ15}, these classifiers exhibit different classification abilities. Hence, we can employ boosting technique~\cite{GradientBoosting:journal/AS/FriedmanJ} to improve the classification ability of weak modality. 

Concretely, assuming the classification performance of the $o$-modality requires improvement, we first apply the gradient boosting algorithm to train $n$ classifiers for the $o$-modality. Since feature extraction focuses on common patterns, we set the encoders of all classifiers to be shared. Then the $j$-th classifier can be defined as: $\Phi^o_t\triangleq \{\theta^o,\Theta^o_t\}, t\in\{1,\cdots, n\}$. In practice, we adopt multiple fully-connected layers and nonlinear activation rectified linear unit~(ReLU)~\cite{ReLU:journals/corr/abs-1803-08375} to construct our classification module. This module called the configurable classifier, is relatively independent and can be adjusted based on the classification ability. Furthermore, we adopt the shared head structure commonly used in MML~\cite{MLA:conf/cvpr/ZhangYBY24,MVIEW:conf/iccv/0065MW023,IGM:conf/ijcai/JiangCY25} to strengthen the interaction between weak and strong modalities during training.

Inspired by gradient boosting~\cite{GradientBoosting:journal/AS/FriedmanJ}, the classification ability can be facilitated through minimizing the residual error introduced by previous classifiers. Concretely, when we learn $t$-th classifier, the residual labels are defined as:
\begin{align}
\hat\y^o_{it}=\y_i-\lambda\sum_{j=1}^{t-1}\y_i\odot\p^o_{ij},\nonumber
\end{align}
where $\lambda\in[0,1]$ is used to soften hard labels~\cite{LabelSmoothing:conf/cvpr/SzegedyVISW16}, $\odot$ denotes the element-wise product, and we utilize $\y_i$ to mask non ground-truth labels to ensure the non-negativity of residual labels. Then the objective function can be defined as follows:
\begin{align}
\epsilon(\x^o_i,\y_i,t)=\ell\big(\p^o_{it},\hat\y^o_{it}\big),\label{eq:residual}
\end{align}
where $\p^o_{it}$ denotes the prediction obtained by $t$-th classifier for $i$-th data point. Since we utilize a shared encoder, the encoder will be updated when training the $t$-th classifier. Therefore, other classifiers must be updated simultaneously to prevent performance degradation. The corresponding objective can be formed as:
\begin{align}
\epsilon_{\mathrm{all}}(\x^o_i,\y_i,t)&=\ell\left(\p^o_{it}+\sum_{j=1}^{t-1}\p^o_{ij},\y_i\right)=\ell\left(\sum_{j=1}^{t}\p^o_{ij},\y_i\right).\label{eq:overall}
\end{align}

Meanwhile, we have to ensure the first $t-1$ classifiers are well-trained. Hence, we define the following objective for $t-1$ classifiers:
\begin{align}
\epsilon_{\mathrm{pre}}(\x^o_i,\y_i,t)&=\ell\left(\sum\nolimits_{j=1}^{t-1}\p^o_{ij},\y_i\right).\label{eq:previous}
\end{align}

By combining~(\ref{eq:residual}),~(\ref{eq:overall}), and~(\ref{eq:previous}), the objective can be defined as:
\begin{align}
L(\x^o_i,\y_i,t)=\epsilon(\x^o_i,\y_i,t)+\epsilon_{\mathrm{all}}(\x^o_i,\y_i,t)+\epsilon_{\mathrm{pre}}(\x^o_i,\y_i,t).\label{eq:obj}
\end{align}
Unlike traditional gradient boosting~\cite{GradientBoosting:journal/AS/FriedmanJ}, our method sustainedly minimizes classification and residual errors by optimizing~(\ref{eq:obj}).
The loss function of sustained boosting can be formed as:
\begin{align}
\LM_{\mathrm{SUB}}(\X^o,\Y,n^o;\Phi^o)=&\frac{1}{N}\sum_{i=1}^{N}L(\x^o_i,\y_i,n^o),\label{obj:gbm}
\end{align}
where $n^o$ denotes the number of classifier for $o$-modality.

\subsection{Adaptive Classifier Assignment}
 Thus far, we have defined a configurable classifier module and designed a sustained boosting in multimodal learning to enhance the classification performance of weak modality. However, recent studies~\cite{OGM:conf/cvpr/PengWD0H22} have shown that differences between modalities evolve dynamically due to imbalance issues in multimodal learning. This implies the need to design a strategy for enhancing classification ability that adapts to dynamic changes. Hence, we propose an adaptive classifier assignment strategy to adjust the number of the weak classifier. For simplicity, we redefine the modality classifiers as: $\Phi^o_t\triangleq \{\theta^o,\Theta^o_t\}, t\in\{1,\cdots, n^o\}$, where $n^o$ is the parameter to be updated. 

Then, we utilize confident score to monitor the learning status. At $t$-th iteration, confident score can be calculated by:
\begin{align}
\forall o\in\{a,v\},s^o_t=\frac{1}{N}\sum\nolimits_{i=1}^N\y_i^\top\left[\sum\nolimits_{j=1}^{n^o}\p^o_{ij}\right].\nonumber
\end{align}
The confident score reflects the classification ability of the models. Hence, if $s^a_t-\sigma s^v_t>\tau$, we assign a new configurable classifier for video modality at this iteration, where $\sigma\ge 1$ is the coefficient. $\tau$ is the dead zone for fault tolerance. On the contrary, we also assign a new configurable classifier for audio modality if $s^a_t-\sigma s^v_t<\tau$. 

Our algorithm is summarized in Algorithm~(\ref{algo:ours}). In practice, we perform adaptive classification assignment strategy to determine if we need to adjust the classification ability every $t_N$ iterations.





\begin{algorithm}[t]
\caption{Learning algorithm of our proposed method.}\label{algo:ours}
\begin{algorithmic}[1]
\Require{Training data $\X$, category labels $\Y$.}
\Ensure{The learned DNN models for all modalities.}\\
\textbf{INIT} Initialize the number of classifier $n^a=1$, $n^v=1$. Initialize iteration $t=1$. Initialize DNN parameters $\Phi^a_t$ and $\Phi^v_t$.
\For{$t=1\mapsto\#iterations$}
\State Sample a mini-batch $\X_t=\{(\x_i^a,\x_i^v)\}_{i=1}^{n_b}$; \Comment{{\color{blue}Learn MML models.}}
\State $\forall \x_i^a,\x_i^v\in\X_t$, calculate features $\u_i^a$ and $\u_i^v$;
\State Calculate predictions $\{\p^a_{ij}\}_{j=1}^{n^a}$ and $\{\p^v_{ij}\}_{j=1}^{n^v}$.
\State Calculate loss in~(\ref{obj:gbm}) based on predictions;
\State Update DNN parameters $\Phi^a_t$ and $\Phi^v_t$ based on SGD;
\If{$\mod(t,t_N)=0$}\Comment{{\color{blue}Adaptive Classification assignment strategy.}}
\State Calculate confident score $\{s^a_t,s^v_t\}$ based on predictions;
\If{$s^a_t-\sigma s^v_t>\tau$}
\State Add a classifier for audio modality;
\State $n^a=n^a+1$;
\ElsIf{$s^a_t-\sigma s^v_t<\tau$}
\State Add a classifier for video modality;
\State $n^v=n^v+1$;
\EndIf
\EndIf
\EndFor
\end{algorithmic}
\end{algorithm}

\subsection{Theoretical Analysis}
The sustained boosting algorithm is introduced to reduce the loss gap between different modalities. In this section, we theoretically analyze its effect on minimizing this gap. We first define the gap function as follows:
\begin{align}
\GM(\Phi)=\LM^a(\Phi^a)-\LM^v(\Phi^v),
\end{align}
where $\Phi^a$ and $\Phi^v$ respectively denote the parameters of audio and video modality, $\LM^a$ and $\LM^v$ denote the overall loss function $\LM_{\mathrm{SUB}}$ for audio and video, respectively.Without loss of generality, we assume that $\LM^a>\LM^v$; the case where $\LM^a<\LM^v$ can be analyzed analogously.

We derive the following conclusions:
\begin{theorem}[Convergence of Gap Loss, Informal]
\label{thm:convergence}
Under some assumptions for the loss function and the effectiveness of sustained boosting algorithm, we have:
\begin{equation}
\GM(\Phi(T))\le\frac{1}{1+\frac{\nu^2\kappa^2}{2L_a\beta^2}T\GM(\Phi(0))}\GM(\Phi(0))
\end{equation}
where $\nu$, $\kappa$, $L_a$ and $\beta$ are constant.
\end{theorem}

The results in Theorem~\ref{thm:convergence} indicate that, by employing the gradient boosting algorithm, the loss gap between modalities converges at a rate of $\OM(1/T)$. The proof can be found in the appendix.

\section{Experiments}
\subsection{Experimental Setup} \label{sec:exp}

{\bf Dataset:} We carry out the experiments on six extensive multimodal datasets, i.e., \CREMAD~\cite{CREMAD:journals/taffco/CaoCKGNV14}, \Kinetics~\cite{Kinetics-Sound:conf/iccv/ArandjelovicZ17}, \NVGesture~\cite{NVGeasture:conf/cvpr/MolchanovYGKTK16}, \VGGSound~\cite{VGGSound:conf/icassp/ChenXVZ20}, \Twitter~\cite{Twitter15:conf/ijcai/Yu019}, and \Sarcasm~\cite{Sarcasm:conf/acl/CaiCW19}~datasets. The \CREMAD, \Kinetics, and \VGGSound~datasets consist of audio and video modalities. \NVGesture~dataset contains three modalities, i.e., RGB, optical flow~(OF), and Depth. \Twitter~and \Sarcasm~datasets consist of image and text modalities. 

To ensure a fair comparison, we strictly follow the data partitioning strategy adopted by representative methods~\cite{OGR-GB:conf/cvpr/WangTF20,OGM:conf/cvpr/PengWD0H22,MLA:conf/cvpr/ZhangYBY24,ReconBoost:conf/icml/CongHua24}. Specifically, the \CREMAD~dataset contains 7,442 clips, which are divided into training set with 6,698 samples and testing set with 744 samples. For \Kinetics~dataset, which contains 19,000 video clips, is divided into training set with 15,000 clips, validation set with 1,900 clips, and testing set with 1,900 clips. \VGGSound~dataset includes 168,618 videos for training and validation, and 13,954 videos for testing. The \NVGesture~dataset is divided into 1,050 samples for training and 482 samples for testing. \Twitter~dataset is divided into training set with 3,197 pairs, validation set with 1,122 pairs and testing set with 1,037 pairs. \Sarcasm~dataset includes 19,816 pairs for the training set, 2,410 pairs for the validation set, and 2,409 pairs for the testing set. More details are provided in the appendix.

{\bf Baselines:} We select two categories of methods for comparison, i.e., traditional multimodal fusion methods and rebalanced multimodal learning methods. In detail, traditional multimodal fusion methods include concatenation fusion~(Concat), affine transformation fusion~(Affine)~\cite{Film:conf/aaai/PerezSVDC18}, multi-layers lstm fusion~(ML-LSTM)~\cite{ML-LSTM:journals/mta/NieYSW21}, prediction summation fusion~(Sum)~\cite{MML:conf/ijcai/YangWZX019}, and prediction weighting fusion~(Weight)~\cite{MML:conf/ijcai/YangWZX019}. Rebalanced multimodal learning methods include MSES~\cite{MSES:conf/acpr/FujimoriEKM19}, G-blend~\cite{OGR-GB:conf/cvpr/WangTF20}, MSLR~\cite{MSLR:conf/acl/YaoM22}, OGM~\cite{OGM:conf/cvpr/PengWD0H22}, PMR~\cite{PMR:conf/cvpr/Fan0WW023}, AGM~\cite{AGM:conf/iccv/LiLHLLZ23}, MMParato~\cite{MMPareto:conf/icml/WeiH24}, SMV~\cite{SMV:conf/cvpr/YakeRZD24}, MLA~\cite{MLA:conf/cvpr/ZhangYBY24}, DI-MML~\cite{DI-MML:conf/mm/FanXWLG24}, LFM~\cite{LFM:conf/nips/0074WJ024}, ReconBoost~\cite{ReconBoost:conf/icml/CongHua24}.


{\bf Evaluation Protocols:} Following the setting of MLA~\cite{MLA:conf/cvpr/ZhangYBY24} and ReconBoost~\cite{ReconBoost:conf/icml/CongHua24}, we adopt accuracy, mean average precision~(MAP) and MacroF1 as evaluation metrics. The accuracy measures the proportion of correct predictions of total predictions. MAP returns the average precision of all samples. And MacroF1 calculates the average F1 across all categories.

{\bf Implementation Details: }Following OGM~\cite{OGM:conf/cvpr/PengWD0H22}, we employ ResNet18~\cite{ResNet:conf/cvpr/HeZRS16} as the backbone to encode audio and video for \CREMAD, \Kinetics~and \VGGSound~datasets. All the parameters of the backbone are randomly initialized. For \NVGesture~dataset, we employ the I3D~\cite{I3D:conf/cvpr/CarreiraZ17} as unimodal branch following the setting of~\cite{nvGesture:conf/icml/WuJCG22}. We initialize the encoder with the pre-trained model trained on ImageNet. For the architecture of the configurable classifier, we explore a two-layer network, which can be denoted as ``Layer1$(D\times 256) \mapsto $ ReLU $ \mapsto $ Layer2 $(256\times K)$''. Here, $D$ denotes the output dimensions of encoders, ``Layer1''/``Layer2'' are fully connected layer, and ``ReLU'' denotes the ReLU~\cite{ReLU:journals/corr/abs-1803-08375} activation layer. Furthermore, the Layer2 is utilized as shared head for all modalities as described in Section~\ref{sec:AGB}. Both Layer1 and Layer2 are randomly initialized. In addition, all hyper-parameters are selected by using the cross-validation strategy. Specifically, we use stochastic gradient descent~(SGD) as the optimizer with a momentum of $0.9$ and weight decay of $1\times 10^{-4}$. The initial learning rate is set to be $1\times 10^{-2}$ for \CREMAD, \Kinetics, \VGGSound~, and \NVGesture~datasets. During training, the learning rate is progressively reduced by a factor of ten upon observing loss saturates. The batch size is set to be 64 for \CREMAD~and \Kinetics~datasets, 16 for \VGGSound~dataset, and 2 for \NVGesture~dataset. We set the iteration $t_N$ for checking whether to assign the classifier to 20 epochs for \CREMAD, 5 for \Twitter, 1 for \Sarcasm, and 10 for \VGGSound, \Kinetics, \NVGesture~datasets. For all datasets, we search $\lambda$ in $\left \{0.1, 0.2, 0.33, 0.5, 1.0 \right \}$. For all datasets, $\sigma$ and $\tau$ are set to be $1.0$ and $0.01$, respectively. For \Twitter~ and \Sarcasm~dataset, following~\cite{Twitter15:conf/ijcai/Yu019,Sarcasm:conf/acl/CaiCW19}, we adopt BERT~\cite{BERT:conf/naacl/DevlinCLT19} as the text encoder and ResNet50~\cite{ResNet:conf/cvpr/HeZRS16} as the image encoder. We use Adam~\cite{Adam:journals/corr/KingmaB14} as the optimizer, with an initial learning rate of $2\times 10^{-5}$. The batch size is set to 32 for \Twitter~and \Sarcasm~datasets. The other parameter settings are the same as audio-video datasets. For comparison methods, the source codes of all baselines are kindly provided by their authors. For fair comparison, all baselines also adopt the same backbone and initialization strategy for the experiment. All experiments are conducted on an NVIDIA GeForce RTX 4090 and all models are implemented with pytorch.


\begin{table*}[t]
\centering
\caption{Classification performance comparison with SOTA baselines. The best and second-best results are highlighted in {\bf bold} and {\underline{underline}}, respectively. The results with gray background denote the performance based on multimodal learning is inferior to that of the best unimodal approach.}
\label{tab:main_results}
\renewcommand\arraystretch{1.3}
\resizebox{\textwidth}{!}{
\begin{tabular}{l|cc|cc|cc|cc|cc|cc}
\toprule
\multirow{2}{*}{\textbf{Method}} &\multicolumn{2}{c|}{\CREMAD} &\multicolumn{2}{c|}{\Kinetics} &\multicolumn{2}{c|}{\VGGSound} &\multicolumn{2}{c|}{\Twitter} &\multicolumn{2}{c|}{\Sarcasm}&\multicolumn{2}{c}{\NVGesture}\\
\cmidrule(lr){2-3} \cmidrule(lr){4-5} \cmidrule(lr){6-7} \cmidrule(lr){8-9} \cmidrule(lr){10-11}\cmidrule(lr){12-13}
& Acc. & MAP    & Acc. & MAP    & Acc. & MAP    & Acc. & F1    & Acc. & F1    & Acc. & F1    \\
\midrule
Unimodal-1     & .4583  & .5879 & .5412  & .5669 & .4655 & .4701 & .5863 & .4333 & .7181 & .7073 & .7822 & .7833 \\
Unimodal-2     & .6317  & .6861 & .5562  & .5837 & .3494 & .3478 & .7367 & .6849 & .8136 & .8056 & .7863 & .7865 \\
Unimodal-3     & N/A    & N/A   & N/A    & N/A   & N/A   & N/A   & N/A   & N/A   & N/A   & N/A   & .8154 & .8183 \\\midrule
Concat         & .6361  & \unimodal{.6841} & .6455  & .7130 & .5116 & .5352 & \unimodal{.7011} & .6386 & .8286 & .8240 & .8237 & .8270 \\
Affine         & .6626  & .7193 & .6424  & .6931 & .5001 & .5155 & \unimodal{.7203} & \unimodal{.5992} & .8240 & .8188 & .8278 & .8281 \\
ML-LSTM        & \unimodal{.6290}  & \unimodal{.6473} & .6394  & .6902 & .4966 & .5139 & \unimodal{.7068} & \unimodal{.6564} & .8277 & .8205 & .8320 & .8330 \\
Sum            & .6344  & .6908 & .6490  & .7103 & .5136 & .5338 & \unimodal{.7300} & \unimodal{.6661} & .8294 & .8247 & \unimodal{.8050} & \unimodal{.8067} \\
Weight         & .6653  & .7134 & .6533  & .7110 & .5144 & .5300 & \unimodal{.7242} & \unimodal{.6516} & .8265 & .8219 & \unimodal{.7842} & \unimodal{.7939} \\\midrule
MSES           & .6546  & .7138 & .6591  & .7196 & .4891 & .5429 & \unimodal{.7252} & \unimodal{.6439} & .8423 & .8369 & \unimodal{.8112} & \unimodal{.8147} \\
G-blend        & .6465  & .7392 & .6722  & .7274 & .5086 & .5555 & \unimodal{.7309} & \unimodal{.6799} & .8286 & .8215 & .8299 & .8305 \\
MSLR           & .6868  & .7412 & .6756  & .7282 & .4987 & .5415 & \unimodal{.7232} & \unimodal{.6382} & .8439 & .8378 & .8237 & .8284 \\
OGM            & .6612  & .7372 & .6582  & .7159 & .4829 & .4978 & \unimodal{.7058} & \unimodal{.6435} & .8360 & .8293 & $-$   & $-$   \\
PMR            & .6659  & .7058 & .6675  & .7274 & .4647 & .4866 & \unimodal{.7357} & \unimodal{.6636} & .8310 & .8256 & $-$   & $-$   \\
AGM            & .6733  & .7807 & .6791  & .7388 & .4711 & .5198 & \unimodal{.7261} & \unimodal{.6502} & .8306 & .8293 & .8279 & .8284 \\
MMParato       & .7487  & .8535 & .7000  & .7850 & .5125 & .5473 & \unimodal{.7358} & \unimodal{.6729} & .8348 & .8284 & .8382 & .8424 \\
SMV            & .7872  & .8417 & .6900  & .7426 & .5031 & .5362 & .7428 & \unimodal{.6817} & .8418 & .8368 & .8352 & .8341 \\
MLA            & .7943  & .8572 & .7004  & \underline{.7945} & .5165 & .5473 & \unimodal{.7352} & \unimodal{.6713} & .8426 & .8348 & .8340 & .8372 \\
DI-MML         & .8158  & .8592 & .7203  & .7426 & .5173 & .5479 & \unimodal{.7248} & \unimodal{.6686} & .8411 & .8315 & $-$   & $-$   \\
LFM            & \underline{.8362}  & \underline{.9006} & \underline{.7253}  & .7897 & \underline{.5274} & \underline{.5694} & \underline{.7501} & \bf.7057 & \underline{.8497} & \underline{.8457} & \underline{.8436} & \underline{.8468} \\
ReconBoost     & .7557  & .8140 & .6855  & .7662 & .5097 & .5387 & .7442 & \unimodal{.6832} & .8437 & .8317 & .8386 & .8434\\\midrule
Ours            & \makecell{\bf.8515 \\ \std{0.0027}}  & \makecell{\bf.9103 \\ \std{0.0014}} & \makecell{\bf.7263 \\ \std{0.0031}}  & \makecell{\bf.7901 \\ \std{0.0063}} & \makecell{\bf.5301 \\ \std{0.0006}} & \makecell{\bf.5826 \\ \std{0.0017}} & \makecell{\bf.7512 \\ \std{0.0068}}   & \makecell{\underline{.6962} \\ \std{0.0016}} & \makecell{\bf.8510 \\ \std{0.0054}} & \makecell{\bf.8458 \\ \std{0.0047}} & \makecell{\bf.8501 \\ \std{0.0025}} & \makecell{\bf.8533 \\ \std{0.0031}}\\
\bottomrule
\end{tabular}
}
\end{table*}

\subsection{Main Results}

{\bf Classification Performance Comparison:} The classification results on all datasets are reported in Table~\ref{tab:main_results}, where ``$-$'' denotes that corresponding methods cannot applied to the dataset with more than two modalities. And Unimodal-1/2/3 is used to denote the results based on unimodal. Unimodal-1/2 respectively denote the video/audio for \CREMAD~and \Kinetics, and text/image for \Twitter~and \Sarcasm. Unimodal-1/2/3 denotes the RGB/OF/Depth modality for \NVGesture~dataset, respectively. Furthermore, the results with a gray background indicate that the performance based on multimodal learning is inferior to that of the best unimodal approach. We can draw the following observations: (1). Compared with unimodal baselines, traditional multimodal fusion methods and rebalanced multimodal learning methods can achieve better performance in almost all cases; (2). Our method can outperform existing SOTA baselines to achieve the best accuracy in all cases for multimodal situations; (2). The accuracy on \NVGesture~dataset demonstrates that our method can extend to the case with more than two modalities and achieve the best performance.

\subsection{Sensitivity to Hyperparameter}
We explore the influence of $\sigma$ and $\lambda$ on \CREMAD~and \Kinetics~datasets in this section. More results can be found in appendix.

{\bf Sensitivity to Threshold $\sigma$: }We study the influence of threshold $\sigma$ on \CREMAD~and \Kinetics~datasets. The accuracy with different $\sigma\in[1,1.75]$ is shown in Figure~\ref{fig:sensitivity}. We can find that our method is not sensitive to threshold $\sigma$ in a large range.

{\bf Sensitivity to Smoothing Factor $\lambda$: }We explore the influence of smoothing factor $\lambda$ on \CREMAD~and \Kinetics~datasets. The accuracy with different $\lambda\in[0.1,1]$ is reported in Figure~\ref{fig:sensitivity}. We can find that our method is not sensitive to hyper-parameter smoothing factor $\lambda$ in a large range.

\begin{figure*}[t] 
\begin{minipage}[t]{0.55\linewidth}\centering
\includegraphics[width=.465\linewidth]{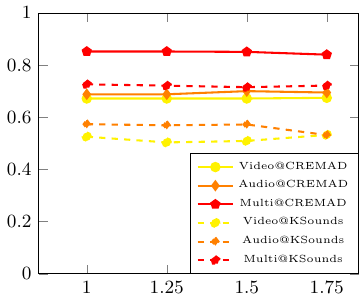}
\includegraphics[width=.465\linewidth]{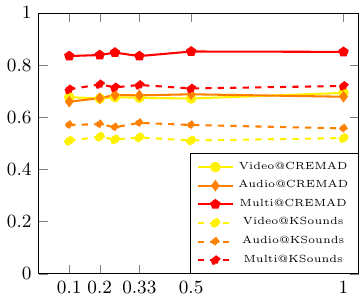} \\
\caption{Sensitivity to $\sigma$~(left) and $\lambda$~(right).}
\label{fig:sensitivity}
\end{minipage}
\hfill
\begin{minipage}[t]{0.45\linewidth}\centering
\includegraphics[width=.95\linewidth]{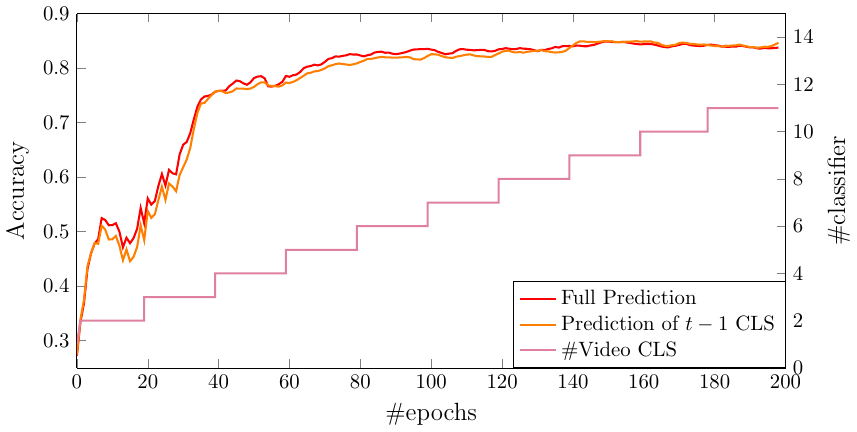} \\
\caption{Performance comparison.}
\label{fig:residual}
\end{minipage}
\end{figure*}

\subsection{Ablation Study}
We investigate the effectiveness of our method by analyzing the influence of the key components of our objectives in Equation~(\ref{eq:residual}),~(\ref{eq:overall}), and~(\ref{eq:previous}), respectively denoted as $\epsilon$, $\epsilon_o$, and $\epsilon_p$. The accuracy results on \CREMAD~and \Kinetics~datasets are reported in Table~\ref{tab:ablation}. From Table~\ref{tab:ablation}, we can find that: (1). Both objectives in Equation~(\ref{eq:residual}),~(\ref{eq:overall}), and~(\ref{eq:previous}) can boost multimodal performance; (2). While the unimodal performance of the method using all objectives may not always reach the highest level, it achieves a more balanced classification performance across modalities. More results are reported in appendix.

We further investigate the impact of residual learning on classification performance by comparing the performance of all $t$ classifiers with that of the first $t-1$ classifiers during the training. The results are presented in Figure~\ref{fig:residual}, where the former accuracy is denoted as ``Full Prediction'' and the latter is denoted as ``Prediction of $t-1$ CLS''. In Figure~\ref{fig:residual}, we also present the number of the video classifier. We observe that the number of classifiers for the video modality has increased, and the performance of all $t$ classifiers is generally superior to that of the first $t-1$ classifiers. This performance gain arises from our learning of the residual objective.

\begin{table}[t]
\centering
\begin{tabular}{cc}
\begin{minipage}{.515\linewidth}
\caption{The accuracy~(Multi/Audio/Video) for ablation study on \CREMAD~and \Kinetics~datasets.}
\label{tab:ablation}
\centering
\renewcommand\arraystretch{1.3}
\resizebox{\textwidth}{!}{
\begin{tabular}{ccc|c|c}
\toprule
$\epsilon$&$\epsilon_{o}$&$\epsilon_{p}$& \CREMAD  & \Kinetics   \\\midrule
\tikzcmark  &\tikzxmark   &\tikzcmark& 0.8333/0.6465/0.6734            & 0.6942/{\bf0.5303}/0.5044 \\
\tikzxmark  &\tikzcmark   &\tikzxmark& 0.8320/0.6573/0.6707            & 0.7190/{\bf0.5303}/0.5145\\
\tikzxmark  &\tikzcmark   &\tikzcmark& 0.8360/{\bf 0.6841}/0.6371      & 0.7198/0.5261/0.5238\\
\tikzcmark  &\tikzcmark   &\tikzcmark& {\bf0.8515}/0.6835/{\bf 0.6828} & {\bf0.7263}/0.5257/{\bf0.5736}\\
\bottomrule
\end{tabular}
}
\end{minipage} &
\begin{minipage}{.415\linewidth}
\caption{The impact of weak classifier selection strategy.}  
\label{tab:selection-strategy}
\centering
\renewcommand\arraystretch{1.3}
\resizebox{\textwidth}{!}{
\begin{tabular}{r|cc|ccc}
\toprule
\multirow{2}{*}{Strategy} &\multicolumn{2}{c|}{\#Classifier} &\multicolumn{3}{c}{{Accuracy}}  \\\cmidrule(lr){2-6}
&          Audio&Video& Multi&Audio   & Video  \\\midrule
Fixed      & 1&10 & {0.8091} & {0.6774} & {0.6156} \\
Fixed      & 1&12 & {0.8118} & {0.6519} & {0.6277} \\\midrule
Adaptive   & 1&10 & \bf{0.8515} & \bf{0.6835} & \bf{0.6828} \\
\bottomrule
\end{tabular}
}
\end{minipage} 
\end{tabular}
\end{table}

\subsection{Further Analysis}\label{sec:further-analysis}
{\bf Impact of Weak Classifier Assignment Strategy:} We conduct an experiment to study the influence of adaptive classifier assignment strategy. Specifically, we design a fixed classifier assignment strategy for comparison. This approach allocates $n^{(\text{fix})}$ classifiers for weak modality during the init stage. And we no longer dynamically adjust the number of classifiers during training for weak modality. 

\begin{figure*}[t] 
\centering
\begin{minipage}{.32\linewidth}
\centering
\includegraphics[width=\linewidth]{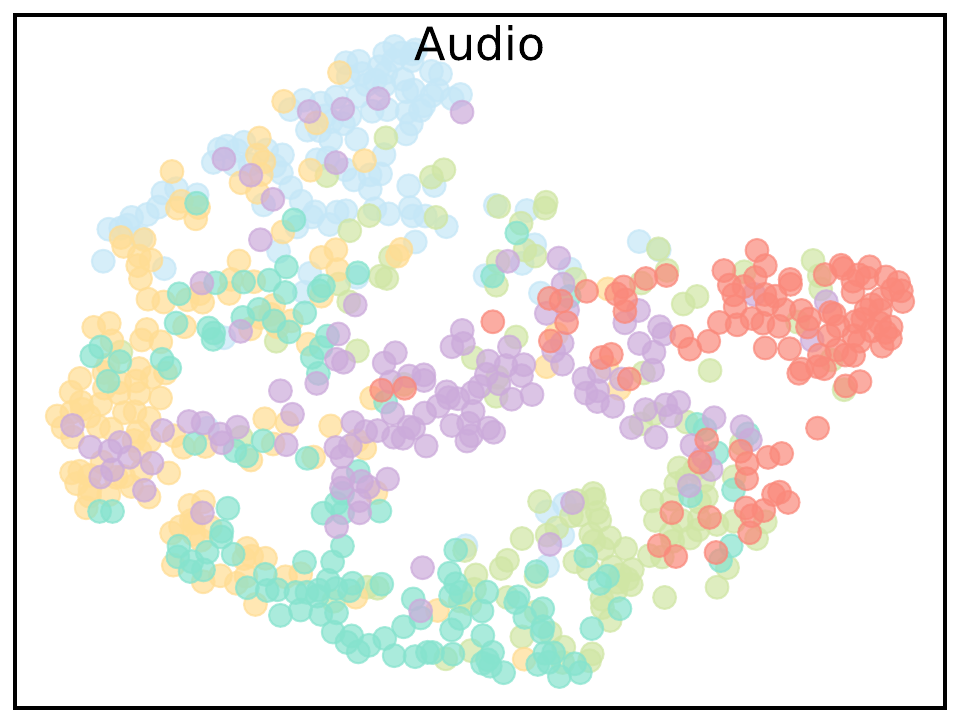}\\
{(a). naive MML@Audio.}
\end{minipage} 
\begin{minipage}{.32\linewidth}
\centering
\includegraphics[width=\linewidth]{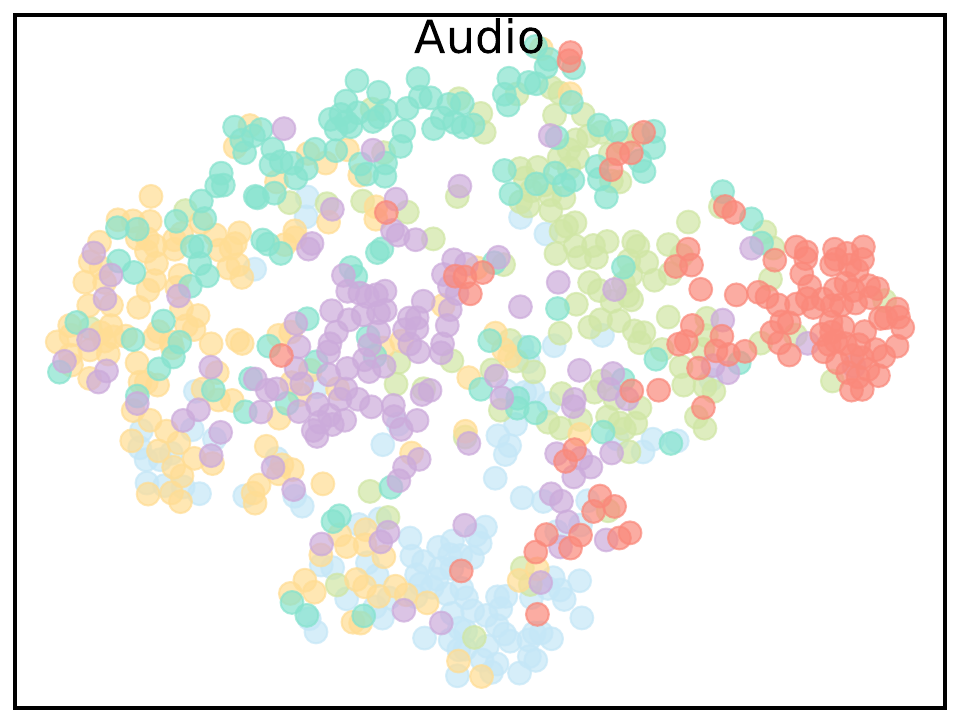}\\
{(b). ReconBoost@Audio.}
\end{minipage}
\begin{minipage}{.32\linewidth}
\centering
\includegraphics[width=\linewidth]{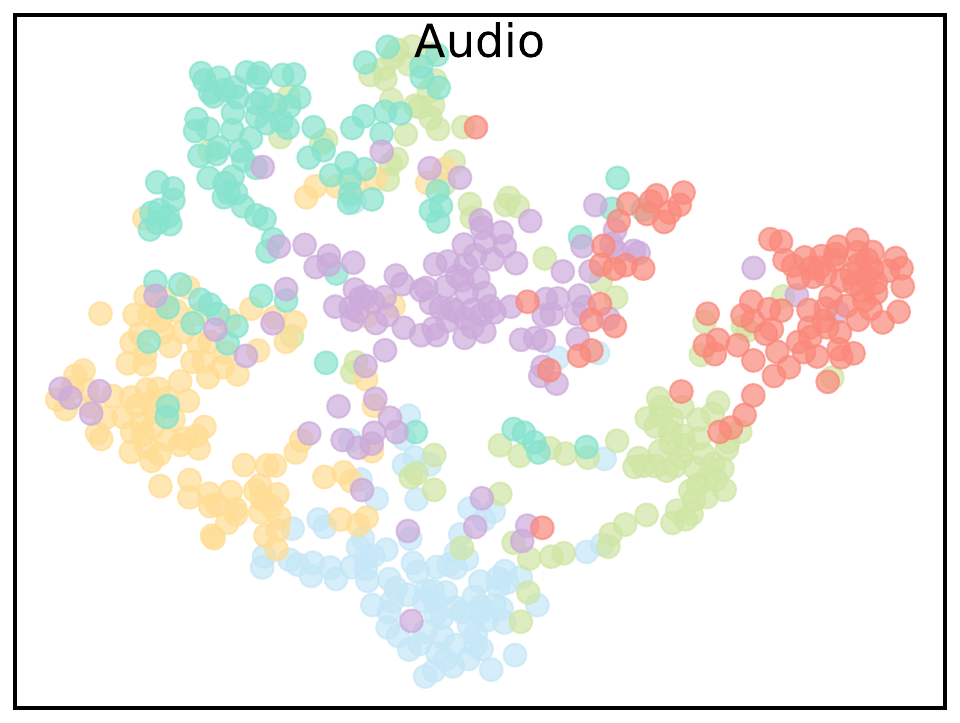}\\
{(c). Ours@Audio.}
\end{minipage}\\
\begin{minipage}{.32\linewidth}
\centering
\includegraphics[width=\linewidth]{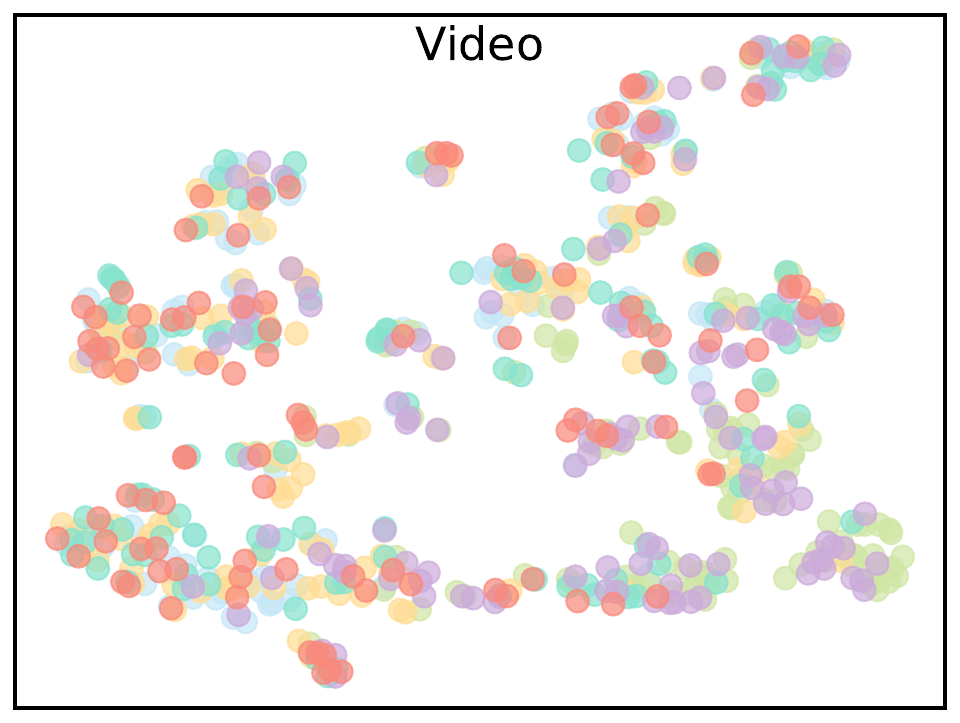}\\
{(d). naive MML@Video.}
\end{minipage} 
\begin{minipage}{.32\linewidth}
\centering
\includegraphics[width=\linewidth]{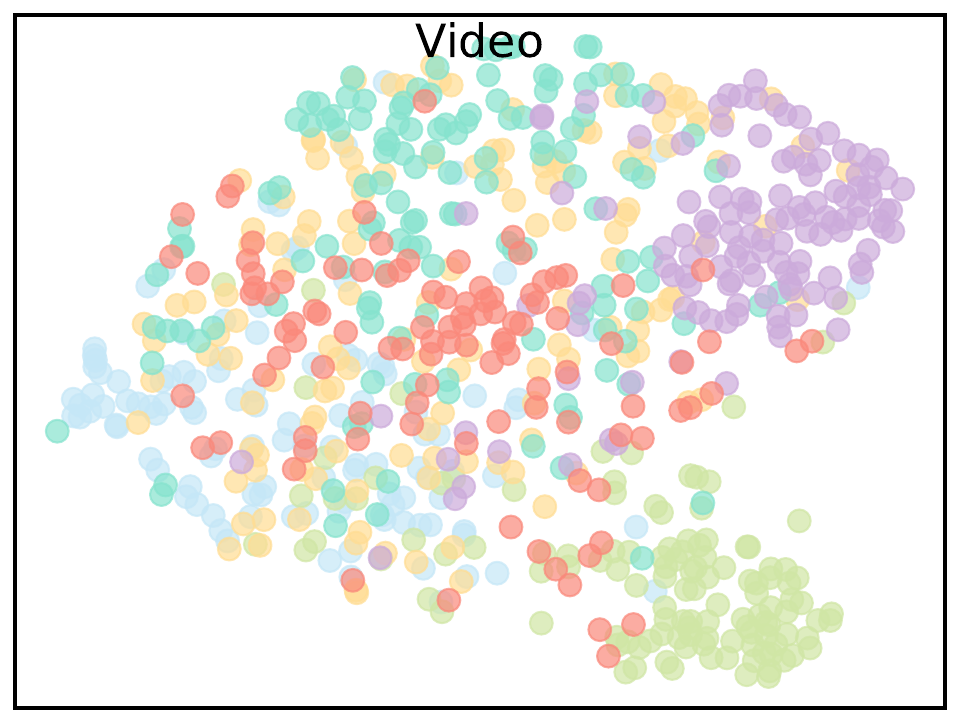}\\
{(e). ReconBoost@Video.}
\end{minipage}
\begin{minipage}{.32\linewidth}
\centering
\includegraphics[width=\linewidth]{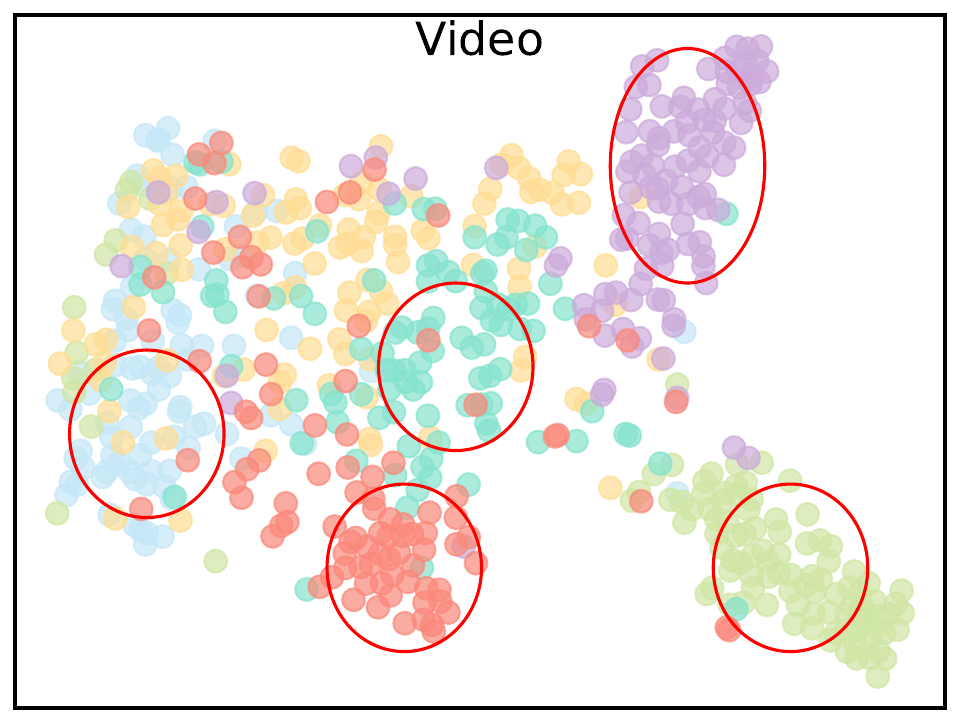}\\
{(f). Ours@Video.}
\end{minipage}
\caption{Visualization on \CREMAD~dataset. The video visualization highlights the need to improve weak modality classification.}
\label{fig:visualization}
\end{figure*}
  
The results on \CREMAD~dataset are reported in Table~\ref{tab:selection-strategy}, where $n^{(\text{fix})}$ is set to be 10 and 12. The results in Table~\ref{tab:selection-strategy} demonstrate that our proposed adaptive classifier assignment strategy can boost performance compared with fixed classifier strategy. This is because our method dynamically adjusts modality classification performance in response to modality imbalance during training.

{\bf Visualization Results:} We further study the property of embeddings through visualization. Specifically, we illustrate the t-SNE~\cite{tSNE:journal/jmlr/MaatenH08} results on \CREMAD~dataset for naive multimodal learning~(naive MML), ReconBoost~\cite{ReconBoost:conf/icml/CongHua24}, and our method in Figure~\ref{fig:visualization}. From Figure~\ref{fig:visualization}, we can find that: (1). Compared to naive MML, our method and ReconBoost can learn more discriminative multimodal features, as both approaches enhance the weak modality using information from the strong modality; (2). Compared to ReconBoost, our method demonstrates significantly superior classification performance on the video modality, with several distinct categories highlighted by circle markers in Figure~\ref{fig:visualization}~(f). This improvement is primarily attributed to our explicit enhancement of the classification capabilities of the weaker modality.

\begin{figure*}[t] 
\begin{minipage}[t]{0.33\linewidth}\centering
\includegraphics[width=\linewidth]{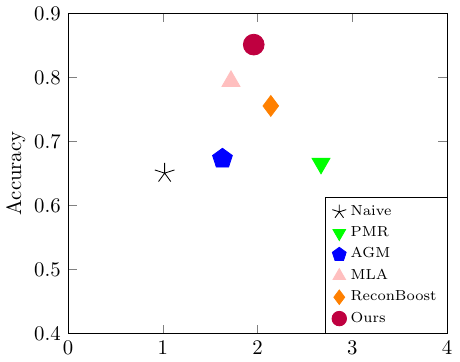} \\
\caption{Training time (hrs).}
\label{fig:train-time}
\end{minipage}
\begin{minipage}[t]{0.33\linewidth}\centering
\includegraphics[width=\linewidth]{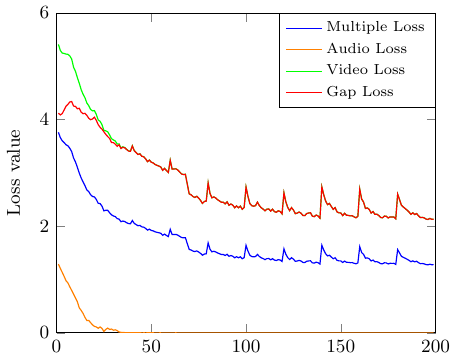} \\
\caption{Loss visualization.}
\label{fig:loss}
\end{minipage}
\begin{minipage}[t]{0.33\linewidth}\centering
\includegraphics[width=\linewidth]{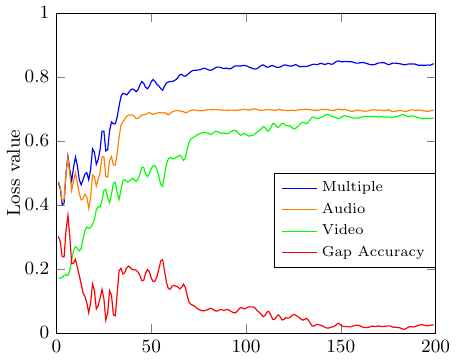} \\
\caption{Accuracy change.}
\label{fig:accuracy}
\end{minipage}
\end{figure*}

\begin{table}[t]
\centering
\begin{tabular}{cc}
\begin{minipage}{.425\linewidth}
\caption{The impact of model capacity.}  
\label{tab:model-capacity}
\centering
\renewcommand\arraystretch{1.3}
\resizebox{\textwidth}{!}{
\begin{tabular}{r|cc|cc|c}
\toprule
\multirow{2}{*}{Method} &\multicolumn{2}{c|}{Arch.} &\multicolumn{2}{c|}{\#params.} &\multirow{2}{*}{Acc.} \\\cmidrule{2-3}\cmidrule{4-5}
&  Audio        &Video  &  Audio        &Video  &\\\midrule
Unimodal    & -  & R18 & - & 11.8M & 0.4718\\
Unimodal    & -  & R34 & - & 23.3M & 0.4731\\\midrule
Naive    & R18  & R18 & 11.8M & 11.8M & \second{0.6507}\\
Naive    & R18  & R34 & 11.8M & 23.3M & 0.6277\\\midrule
Ours     & R18  & R18 & 11.8M & 12.8M & \first{0.8515}\\
\bottomrule
\end{tabular}
}
\end{minipage} &
\begin{minipage}{.5\linewidth}
\caption{Performance comparison in scenario with modality missing.}  
\label{tab:modality-missing}
\centering
\renewcommand\arraystretch{1.3}
\resizebox{\textwidth}{!}{
\begin{tabular}{r|cccc}
\toprule
Method     & $rate=0\%$  &  $rate=20\%$  & $rate=50\%$  \\\midrule
Naive MML  & 0.6507      &   0.5849      & 0.5242       \\\midrule
ReconBoost & 0.7557	     &   0.6321      & 0.5568       \\
MLA        & 0.7943      &   0.6935      & 0.5753       \\\midrule
Ours       &\bf 0.8515	 & \bf  0.7540   & \bf 0.6008   \\
\bottomrule
\end{tabular}
}
\end{minipage} 
\end{tabular}
\end{table}

{\bf Training Overhead: }Considering that our method introduces additional computational overhead due to the use of the sustained boosting strategy, we compare its training cost with that of competitive SOTA baselines, including Naive, PMR, AGM, MLA, and ReconBoost, through empirical experiments under the same setting. The results are shown in Figure~\ref{fig:train-time}. It can be observed that our method achieves the best accuracy while maintaining competitive training time.

{\bf Convergence of Gap Loss: }To further validate the convergence of the gap function, we conducted experiments on the \CREMAD~dataset. The change of gap function $\GM(\Phi)$ during training process is reported in Figure~\ref{fig:loss}, where we also report the unimodal loss and multimodal loss. It can be observed that as training progresses, the gap function gradually converges. Moreover, we report the accuracy changes during training in Figure~\ref{fig:accuracy}. From Figure~\ref{fig:accuracy}, it can be seen that the accuracy exhibits a similar convergence trend.

{\bf Impact of the Model Capacity: }Since our method utilizes multiple designed classifiers for weak modality, it essentially utilizes more model parameters than baselines. Does our method benefit solely from having more model parameters? We conduct an experiment to verify this. Specifically, we replace backbone ResNet18~(R18) of weak modality video as a larger backbone, i.e., ResNet34~(R34). Then we run the baselines and our method on \CREMAD~dataset. The results are shown in Table~\ref{tab:model-capacity}. From Table~\ref{tab:model-capacity}, we can find that our approach improves accuracy by nearly 20\% with only an additional 1M parameters compared to naive MML with the same network architecture. Our method also outperforms the naive MML baseline with larger backbone. Furthermore, we observe an interesting and counterintuitive phenomenon. That is, the method with ResNet34 is worse than that with ResNet18. The reason behind this may be that the ResNet34-based method is more difficult to converge.

{\bf Stability under Modality Missing: }Our method is adaptable to scenarios involving missing modalities. We comprehensively evaluate its performance under test-time missing modality conditions. Specifically, test-time missing refers to cases where the modalities are complete during training but missing during the testing phase. The experiments are conducted on \CREMAD~dataset with different missing rate~\cite{MLA:conf/cvpr/ZhangYBY24}. Naive MML, ReconBoost~\cite{ReconBoost:conf/icml/CongHua24}, and MLA~\cite{MLA:conf/cvpr/ZhangYBY24} are selected as baselines for comparison, where MLA introduces specifically designed algorithms to address the corresponding challenges in modality missing. We report the results with missing rate 20\% and 50\% in Table~\ref{tab:modality-missing}. We can find that as the modality missing rate increases, the performance of all methods declines. Nevertheless, our method consistently achieves the best performance under all missing rates, demonstrating the effectiveness of our method in scenario with missing modality.

\section{Conclusion} \label{sec:conclusion}
To address the modality imbalance issue, we propose a novel multimodal learning approach by designing a sustained boosting algorithm to dynamically enhance the classification ability of weak modality. Concretely, we first propose a sustained boosting algorithm for multimodal learning by minimizing the classification and residual errors simultaneously. Then, we propose an adaptive classifier assignment strategy to dynamically facilitate the classification ability of weak modality. The effectiveness of the proposed boosting algorithm is theoretically guaranteed by analyzing the convergence properties of the cross-modal gap function. To this end, the classification ability can be rebalanced adaptively during the training procedure. Experiments on widely used datasets reveal that our proposed method can achieve state-of-the-art performance compared with various baselines by a large margin.

{\bf Limitations: }Our proposed method mainly focuses on the classifier of each modality. For early fusion MML, our method can extend to balance the strong, weak, and fusion classification abilities. In addition, our theoretical analysis only examines the effect of the boosting algorithm on the convergence of the cross-modal gap function. In the full iterative framework, the overall convergence behavior and its influence on the model’s learning capability warrant further investigation. We leave a more comprehensive investigation as future work.

\begin{ack}
This work was supported in part by the National Key R\&D Program of China (2022YFF0712100), in part by the NSFC (62276131, 62506168), in part by the Natural Science Foundation of Jiangsu Province of China under Grant (BK20240081, BG2024042, BK20251431), and in part by the Fundamental Research Funds for the Central Universities (No.30925010205).
\end{ack}

{\small
\bibliography{ref}
\bibliographystyle{abbrv}
}

\newpage
\section*{Reproducibility Checklist}

\begin{enumerate}

\item {\bf Claims}
    \item[] Question: Do the main claims made in the abstract and introduction accurately reflect the paper's contributions and scope?
    \item[] Answer: \answerYes{} 
    \item[] Justification:  Main contributions and scope were reflected in the abstract and introduction.
    \item[] Guidelines:
    \begin{itemize}
        \item The answer NA means that the abstract and introduction do not include the claims made in the paper.
        \item The abstract and/or introduction should clearly state the claims made, including the contributions made in the paper and important assumptions and limitations. A No or NA answer to this question will not be perceived well by the reviewers. 
        \item The claims made should match theoretical and experimental results, and reflect how much the results can be expected to generalize to other settings. 
        \item It is fine to include aspirational goals as motivation as long as it is clear that these goals are not attained by the paper. 
    \end{itemize}

\item {\bf Limitations}
    \item[] Question: Does the paper discuss the limitations of the work performed by the authors?
    \item[] Answer: \answerYes{} 
    \item[] Justification:  Limitations of the work was discussed in Section \ref{sec:conclusion}.
    \item[] Guidelines:
    \begin{itemize}
        \item The answer NA means that the paper has no limitation while the answer No means that the paper has limitations, but those are not discussed in the paper. 
        \item The authors are encouraged to create a separate "Limitations" section in their paper.
        \item The paper should point out any strong assumptions and how robust the results are to violations of these assumptions (e.g., independence assumptions, noiseless settings, model well-specification, asymptotic approximations only holding locally). The authors should reflect on how these assumptions might be violated in practice and what the implications would be.
        \item The authors should reflect on the scope of the claims made, e.g., if the approach was only tested on a few datasets or with a few runs. In general, empirical results often depend on implicit assumptions, which should be articulated.
        \item The authors should reflect on the factors that influence the performance of the approach. For example, a facial recognition algorithm may perform poorly when image resolution is low or images are taken in low lighting. Or a speech-to-text system might not be used reliably to provide closed captions for online lectures because it fails to handle technical jargon.
        \item The authors should discuss the computational efficiency of the proposed algorithms and how they scale with dataset size.
        \item If applicable, the authors should discuss possible limitations of their approach to address problems of privacy and fairness.
        \item While the authors might fear that complete honesty about limitations might be used by reviewers as grounds for rejection, a worse outcome might be that reviewers discover limitations that aren't acknowledged in the paper. The authors should use their best judgment and recognize that individual actions in favor of transparency play an important role in developing norms that preserve the integrity of the community. Reviewers will be specifically instructed to not penalize honesty concerning limitations.
    \end{itemize}

\item {\bf Theory assumptions and proofs}
    \item[] Question: For each theoretical result, does the paper provide the full set of assumptions and a complete (and correct) proof?
    \item[] Answer: \answerYes{} 
    \item[] Justification:  The theories and Corollaries presented in the paper are supported by proofs provided in Appendix~\ref{sec:Proof}.
    \item[] Guidelines:
    \begin{itemize}
        \item The answer NA means that the paper does not include theoretical results. 
        \item All the theorems, formulas, and proofs in the paper should be numbered and cross-referenced.
        \item All assumptions should be clearly stated or referenced in the statement of any theorems.
        \item The proofs can either appear in the main paper or the supplemental material, but if they appear in the supplemental material, the authors are encouraged to provide a short proof sketch to provide intuition. 
        \item Inversely, any informal proof provided in the core of the paper should be complemented by formal proofs provided in appendix or supplemental material.
        \item Theorems and Lemmas that the proof relies upon should be properly referenced. 
    \end{itemize}

    \item {\bf Experimental result reproducibility}
    \item[] Question: Does the paper fully disclose all the information needed to reproduce the main experimental results of the paper to the extent that it affects the main claims and/or conclusions of the paper (regardless of whether the code and data are provided or not)?
    \item[] Answer: \answerYes{} 
    \item[] Justification: As detailed in Section \ref{sec:exp}, we provide multimodal datasets and introduce the baseline methods, evaluation metrics, and implementation details.
    \item[] Guidelines:
    \begin{itemize}
        \item The answer NA means that the paper does not include experiments.
        \item If the paper includes experiments, a No answer to this question will not be perceived well by the reviewers: Making the paper reproducible is important, regardless of whether the code and data are provided or not.
        \item If the contribution is a dataset and/or model, the authors should describe the steps taken to make their results reproducible or verifiable. 
        \item Depending on the contribution, reproducibility can be accomplished in various ways. For example, if the contribution is a novel architecture, describing the architecture fully might suffice, or if the contribution is a specific model and empirical evaluation, it may be necessary to either make it possible for others to replicate the model with the same dataset, or provide access to the model. In general. releasing code and data is often one good way to accomplish this, but reproducibility can also be provided via detailed instructions for how to replicate the results, access to a hosted model (e.g., in the case of a large language model), releasing of a model checkpoint, or other means that are appropriate to the research performed.
        \item While NeurIPS does not require releasing code, the conference does require all submissions to provide some reasonable avenue for reproducibility, which may depend on the nature of the contribution. For example
        \begin{enumerate}
            \item If the contribution is primarily a new algorithm, the paper should make it clear how to reproduce that algorithm.
            \item If the contribution is primarily a new model architecture, the paper should describe the architecture clearly and fully.
            \item If the contribution is a new model (e.g., a large language model), then there should either be a way to access this model for reproducing the results or a way to reproduce the model (e.g., with an open-source dataset or instructions for how to construct the dataset).
            \item We recognize that reproducibility may be tricky in some cases, in which case authors are welcome to describe the particular way they provide for reproducibility. In the case of closed-source models, it may be that access to the model is limited in some way (e.g., to registered users), but it should be possible for other researchers to have some path to reproducing or verifying the results.
        \end{enumerate}
    \end{itemize}

\item {\bf Open access to data and code}
    \item[] Question: Does the paper provide open access to the data and code, with sufficient instructions to faithfully reproduce the main experimental results, as described in supplemental material?
    \item[] Answer: \answerYes{} 
    \item[] Justification: We release our code at \url{https://github.com/njustkmg/NeurIPS25-AUG}. Furthermore, all datasets we used in this paper are available online based on their corresponding paper.
    \item[] Guidelines:
    \begin{itemize}
        \item The answer NA means that paper does not include experiments requiring code.
        \item Please see the NeurIPS code and data submission guidelines (\url{https://nips.cc/public/guides/CodeSubmissionPolicy}) for more details.
        \item While we encourage the release of code and data, we understand that this might not be possible, so “No” is an acceptable answer. Papers cannot be rejected simply for not including code, unless this is central to the contribution (e.g., for a new open-source benchmark).
        \item The instructions should contain the exact command and environment needed to run to reproduce the results. See the NeurIPS code and data submission guidelines (\url{https://nips.cc/public/guides/CodeSubmissionPolicy}) for more details.
        \item The authors should provide instructions on data access and preparation, including how to access the raw data, preprocessed data, intermediate data, and generated data, etc.
        \item The authors should provide scripts to reproduce all experimental results for the new proposed method and baselines. If only a subset of experiments are reproducible, they should state which ones are omitted from the script and why.
        \item At submission time, to preserve anonymity, the authors should release anonymized versions (if applicable).
        \item Providing as much information as possible in supplemental material (appended to the paper) is recommended, but including URLs to data and code is permitted.
    \end{itemize}

\item {\bf Experimental setting/details}
    \item[] Question: Does the paper specify all the training and test details (e.g., data splits, hyperparameters, how they were chosen, type of optimizer, etc.) necessary to understand the results?
    \item[] Answer: \answerYes{} 
    \item[] Justification: All implementation details were provided in Section \ref{sec:exp}.
    \item[] Guidelines:
    \begin{itemize}
        \item The answer NA means that the paper does not include experiments.
        \item The experimental setting should be presented in the core of the paper to a level of detail that is necessary to appreciate the results and make sense of them.
        \item The full details can be provided either with the code, in appendix, or as supplemental material.
    \end{itemize}

\item {\bf Experiment statistical significance}
    \item[] Question: Does the paper report error bars suitably and correctly defined or other appropriate information about the statistical significance of the experiments?
    \item[] Answer: \answerYes{} 
    \item[] Justification: All results in this paper are obtained by conducting experiments on three random seeds.
    \item[] Guidelines:
    \begin{itemize}
        \item The answer NA means that the paper does not include experiments.
        \item The authors should answer "Yes" if the results are accompanied by error bars, confidence intervals, or statistical significance tests, at least for the experiments that support the main claims of the paper.
        \item The factors of variability that the error bars are capturing should be clearly stated (for example, train/test split, initialization, random drawing of some parameter, or overall run with given experimental conditions).
        \item The method for calculating the error bars should be explained (closed form formula, call to a library function, bootstrap, etc.)
        \item The assumptions made should be given (e.g., Normally distributed errors).
        \item It should be clear whether the error bar is the standard deviation or the standard error of the mean.
        \item It is OK to report 1-sigma error bars, but one should state it. The authors should preferably report a 2-sigma error bar than state that they have a 96\% CI, if the hypothesis of Normality of errors is not verified.
        \item For asymmetric distributions, the authors should be careful not to show in tables or figures symmetric error bars that would yield results that are out of range (e.g. negative error rates).
        \item If error bars are reported in tables or plots, The authors should explain in the text how they were calculated and reference the corresponding figures or tables in the text.
    \end{itemize}

\item {\bf Experiments compute resources}
    \item[] Question: For each experiment, does the paper provide sufficient information on the computer resources (type of compute workers, memory, time of execution) needed to reproduce the experiments?
    \item[] Answer: \answerYes{} 
    \item[] Justification: The computer resource information for all experiments was provided in Section \ref{sec:exp}. 
    \item[] Guidelines:
    \begin{itemize}
        \item The answer NA means that the paper does not include experiments.
        \item The paper should indicate the type of compute workers CPU or GPU, internal cluster, or cloud provider, including relevant memory and storage.
        \item The paper should provide the amount of compute required for each of the individual experimental runs as well as estimate the total compute. 
        \item The paper should disclose whether the full research project required more compute than the experiments reported in the paper (e.g., preliminary or failed experiments that didn't make it into the paper). 
    \end{itemize}
    
\item {\bf Code of ethics}
    \item[] Question: Does the research conducted in the paper conform, in every respect, with the NeurIPS Code of Ethics \url{https://neurips.cc/public/EthicsGuidelines}?
    \item[] Answer: \answerYes{} 
    \item[] Justification: We have checked the NeurIPS code of ethics for our paper.
    \item[] Guidelines:
    \begin{itemize}
        \item The answer NA means that the authors have not reviewed the NeurIPS Code of Ethics.
        \item If the authors answer No, they should explain the special circumstances that require a deviation from the Code of Ethics.
        \item The authors should make sure to preserve anonymity (e.g., if there is a special consideration due to laws or regulations in their jurisdiction).
    \end{itemize}

\item {\bf Broader impacts}
    \item[] Question: Does the paper discuss both potential positive societal impacts and negative societal impacts of the work performed?
    \item[] Answer: \answerNA{} 
    \item[] Justification: Our paper does not deal with this aspect of the problem.
    \item[] Guidelines:
    \begin{itemize}
        \item The answer NA means that there is no societal impact of the work performed.
        \item If the authors answer NA or No, they should explain why their work has no societal impact or why the paper does not address societal impact.
        \item Examples of negative societal impacts include potential malicious or unintended uses (e.g., disinformation, generating fake profiles, surveillance), fairness considerations (e.g., deployment of technologies that could make decisions that unfairly impact specific groups), privacy considerations, and security considerations.
        \item The conference expects that many papers will be foundational research and not tied to particular applications, let alone deployments. However, if there is a direct path to any negative applications, the authors should point it out. For example, it is legitimate to point out that an improvement in the quality of generative models could be used to generate deepfakes for disinformation. On the other hand, it is not needed to point out that a generic algorithm for optimizing neural networks could enable people to train models that generate Deepfakes faster.
        \item The authors should consider possible harms that could arise when the technology is being used as intended and functioning correctly, harms that could arise when the technology is being used as intended but gives incorrect results, and harms following from (intentional or unintentional) misuse of the technology.
        \item If there are negative societal impacts, the authors could also discuss possible mitigation strategies (e.g., gated release of models, providing defenses in addition to attacks, mechanisms for monitoring misuse, mechanisms to monitor how a system learns from feedback over time, improving the efficiency and accessibility of ML).
    \end{itemize}
    
\item {\bf Safeguards}
    \item[] Question: Does the paper describe safeguards that have been put in place for responsible release of data or models that have a high risk for misuse (e.g., pretrained language models, image generators, or scraped datasets)?
    \item[] Answer: \answerNA{} 
    \item[] Justification: Our paper does not deal with this aspect of the problem.
    \item[] Guidelines:
    \begin{itemize}
        \item The answer NA means that the paper poses no such risks.
        \item Released models that have a high risk for misuse or dual-use should be released with necessary safeguards to allow for controlled use of the model, for example by requiring that users adhere to usage guidelines or restrictions to access the model or implementing safety filters. 
        \item Datasets that have been scraped from the Internet could pose safety risks. The authors should describe how they avoided releasing unsafe images.
        \item We recognize that providing effective safeguards is challenging, and many papers do not require this, but we encourage authors to take this into account and make a best faith effort.
    \end{itemize}

\item {\bf Licenses for existing assets}
    \item[] Question: Are the creators or original owners of assets (e.g., code, data, models), used in the paper, properly credited and are the license and terms of use explicitly mentioned and properly respected?
    \item[] Answer: \answerYes{} 
    \item[] Justification: The assets (datasets, code, models) used in the paper are open source, and our use follows the relevant protocols.
    \item[] Guidelines:
    \begin{itemize}
        \item The answer NA means that the paper does not use existing assets.
        \item The authors should cite the original paper that produced the code package or dataset.
        \item The authors should state which version of the asset is used and, if possible, include a URL.
        \item The name of the license (e.g., CC-BY 4.0) should be included for each asset.
        \item For scraped data from a particular source (e.g., website), the copyright and terms of service of that source should be provided.
        \item If assets are released, the license, copyright information, and terms of use in the package should be provided. For popular datasets, \url{paperswithcode.com/datasets} has curated licenses for some datasets. Their licensing guide can help determine the license of a dataset.
        \item For existing datasets that are re-packaged, both the original license and the license of the derived asset (if it has changed) should be provided.
        \item If this information is not available online, the authors are encouraged to reach out to the asset's creators.
    \end{itemize}

\item {\bf New assets}
    \item[] Question: Are new assets introduced in the paper well documented and is the documentation provided alongside the assets?
    \item[] Answer: \answerYes{} 
    \item[] Justification: The code we released in this paper contains related documents.
    \item[] Guidelines:
    \begin{itemize}
        \item The answer NA means that the paper does not release new assets.
        \item Researchers should communicate the details of the dataset/code/model as part of their submissions via structured templates. This includes details about training, license, limitations, etc. 
        \item The paper should discuss whether and how consent was obtained from people whose asset is used.
        \item At submission time, remember to anonymize your assets (if applicable). You can either create an anonymized URL or include an anonymized zip file.
    \end{itemize}

\item {\bf Crowdsourcing and research with human subjects}
    \item[] Question: For crowdsourcing experiments and research with human subjects, does the paper include the full text of instructions given to participants and screenshots, if applicable, as well as details about compensation (if any)? 
    \item[] Answer: \answerNA{} 
    \item[] Justification: Our paper doesn’t involve the crowdsourcing experiments and research with human subjects.
    \item[] Guidelines:
    \begin{itemize}
        \item The answer NA means that the paper does not involve crowdsourcing nor research with human subjects.
        \item Including this information in the supplemental material is fine, but if the main contribution of the paper involves human subjects, then as much detail as possible should be included in the main paper. 
        \item According to the NeurIPS Code of Ethics, workers involved in data collection, curation, or other labor should be paid at least the minimum wage in the country of the data collector. 
    \end{itemize}

\item {\bf Institutional review board (IRB) approvals or equivalent for research with human subjects}
    \item[] Question: Does the paper describe potential risks incurred by study participants, whether such risks were disclosed to the subjects, and whether Institutional Review Board (IRB) approvals (or an equivalent approval/review based on the requirements of your country or institution) were obtained?
    \item[] Answer: \answerNA{} 
    \item[] Justification: Our paper doesn’t involve the potential risks.
    \item[] Guidelines:
    \begin{itemize}
        \item The answer NA means that the paper does not involve crowdsourcing nor research with human subjects.
        \item Depending on the country in which research is conducted, IRB approval (or equivalent) may be required for any human subjects research. If you obtained IRB approval, you should clearly state this in the paper. 
        \item We recognize that the procedures for this may vary significantly between institutions and locations, and we expect authors to adhere to the NeurIPS Code of Ethics and the guidelines for their institution. 
        \item For initial submissions, do not include any information that would break anonymity (if applicable), such as the institution conducting the review.
    \end{itemize}

\item {\bf Declaration of LLM usage}
    \item[] Question: Does the paper describe the usage of LLMs if it is an important, original, or non-standard component of the core methods in this research? Note that if the LLM is used only for writing, editing, or formatting purposes and does not impact the core methodology, scientific rigorousness, or originality of the research, declaration is not required.
    \item[] Answer: \answerNo{} 
    \item[] Justification: Not used at all (you can then skip the rest).
    \item[] Guidelines:
    \begin{itemize}
        \item The answer NA means that the core method development in this research does not involve LLMs as any important, original, or non-standard components.
        \item Please refer to our LLM policy (\url{https://neurips.cc/Conferences/2025/LLM}) for what should or should not be described.
    \end{itemize}

\end{enumerate}

\clearpage
\appendix

\renewcommand{\thefigure}{\Roman{figure}}
\renewcommand{\thetable}{\Roman{table}}

\makeatletter
\renewcommand{\theHtable}{A\arabic{table}}
\renewcommand{\theHfigure}{A\arabic{figure}}
\makeatother

\setcounter{figure}{0}
\setcounter{table}{0}

\section{Proof} \label{sec:Proof}
During training process, the gradient boosting strategy will be applied when the performance gap is larger than a threshold at each $c$ iterations. Now we analyze the convergence rate of $\GM$ with gradient boosting strategy.

Furthermore, we make the following assumption:
\begin{assumption}[Strongly Convexity]\label{ass:strong-convex}
We assume that the function $\{\LM^a(\cdot),\LM^v(\cdot)\}$ is $\mu$-strongly convex. That is,
for any ${\Phi}$ we have:
\begin{align}\label{ass:sconvex:eqn}
\forall o\in\{a,v\},\;\LM^o(\Phi^o)\ge\LM^o(\Phi^*)+\frac{\mu_o}{2}\|\Phi^o-\Phi^*\|^2
\end{align}
and
\begin{align}
\forall o\in\{a,v\},\;\|\nabla\LM^o(\Phi(t))\|^2\ge\mu_o|\LM^o(\Phi(t))-\LM^o(\Phi^*)|
\end{align}
\end{assumption}

\begin{assumption}[Smoothness]\label{ass:smooth}
We assume that the functions $\{\GM(\cdot), \LM_a(\cdot),\LM_v(\cdot)\}$ are $L$-smooth. That is,
for any ${\Phi}, {\Phi}'$, we have
\begin{align}\label{ass:smotth:eqn}
\forall o\in\{a,v\},\;\LM^o\left(\Phi\right) -\LM^o\left(\Phi'\right)
 \le\langle \nabla \LM^o\left(\Phi'\right), \Phi - \Phi'\rangle + \frac{L_o}{2}\|\Phi - \Phi'\|^2.
\end{align}
and 
\begin{align}
\GM\left(\Phi\right) -\GM\left(\Phi'\right)
 \le\langle \nabla \GM\left(\Phi'\right), \Phi - \Phi'\rangle + \frac{L_g}{2}\|\Phi - \Phi'\|^2.\label{eq:smooth-GM}
\end{align}
\end{assumption}

\begin{assumption}[Effectiveness of Weak Classifier]\label{ass:classification}
We assume that the gradient boosting strategy for weak modality is effective. That is, there exists a $\nu\in(0,1)$:
\begin{align}
\langle \nabla \LM^a\left(\Phi(t)\right), h(\Phi(t))\rangle\le-\nu\|\nabla \LM^a\left(\Phi(t)\right)\|^2.\label{eq2}
\end{align}
And there exists a constant $\beta>0$:
\begin{align}
\|h(\Phi(t))\|\le\beta\|\nabla \LM^a\left(\Phi(t)\right)\|,\label{eq3}
\end{align}
where $h(\cdot)$ denotes the function obtained by boosting algorithm. $h(\cdot)$ aims to fit the negative gradient of weak modality model.
\end{assumption}

Based on the definition of $h(\cdot)$, the updating rule can be formed as:
\begin{align}
\Phi(t+1)
&=\Phi(t)-\eta\nabla\LM(\Phi(t))\\
&=\Phi(t)+\eta h(\Phi(t))
\end{align}
where $\eta$ denotes the learning rate.

\begin{assumption}[Optimal of Multimodels]\label{ass:optimal}
We assume that there exists an optimal solution $\Phi^*$ so that we have: $\nabla\LM^a(\Phi^*)=\nabla\LM^v(\Phi^*)=0$.
\end{assumption}

\begin{assumption}[Bounded Initial Value of $\GM(\Phi(t))$]\label{ass:initial}
We assume that the initial value of $\GM(\Phi(t))$ is bounded, i.e., $\GM(\Phi(0))\le\infty$.
\end{assumption}

Based on Assumption~(\ref{ass:strong-convex}),~(\ref{ass:optimal}), we have the following lemma:
\begin{lemma}[Gap Bound] There exists a constant $\kappa$ such that the loss gap function $\GM(\cdot)$ and the gradient norm of the strong modality satisfy the following relationship:\label{lemma:gradient-norm}
\begin{align}
\|\nabla \LM^a(\Phi_t)\|\ge\kappa|\GM(\Phi_t)|.
\end{align}
\end{lemma}

\begin{proof}[Proof of Lemma~\ref{lemma:gradient-norm}]
Since $\LM^a(\cdot)$ and $\LM^v(\cdot)$ are strongly convex, there exists an constant $\rho\in(0,1)$ satisfying:
\begin{align}
\|\Phi(t)-\Phi^*\|\le\rho^t\|\Phi(0)-\Phi^*\|
\end{align}

Since the smoothness, we have:
\begin{align}
\forall o\in\{a,v\},\LM^o(\Phi(t))-\LM^o(\Phi^*)
&\le\frac{L^o}{2}\|\Phi(t)-\Phi^*\|^2\\
&\le\frac{L^o}{2}\rho^{2t}\|\Phi(0)-\Phi^*\|^2
\end{align}

Then we have:
\begin{align}
\GM(\Phi(t))
&=\LM^a(\Phi(t))-\LM^v(\Phi(t))\\
&=\LM^a(\Phi(t))-\LM^a(\Phi^*)-(\LM^v(\Phi(t))-\LM^v(\Phi^*))\\
&\le|\LM^a(\Phi(t))-\LM^a(\Phi^*)|+|(\LM^v(\Phi(t))-\LM^v(\Phi^*))|\\
&\le\left(\frac{L_a}{2}+\frac{L_v}{2}\right)\rho^{2t}\|\Phi(0)-\Phi^*\|^2\\
&\le\left(\frac{L_a}{2}+\frac{L_v}{2}\right)\|\Phi(0)-\Phi^*\|^2
\end{align}

By setting $c=\left(\frac{L_a}{2}+\frac{L_v}{2}\right)\|\Phi(0)-\Phi^*\|^2$, we have:
\begin{align}
|\GM(\Phi(t))|\le c
\end{align}

Then we have:
\begin{align}
|\GM(\Phi(t))|\ge\frac{1}{c}|\GM(\Phi(t))|^2
\end{align}

According to Assumption~(\ref{ass:strong-convex}), we have:
\begin{align}
\|\nabla\LM^a(\Phi(t))\|^2\ge\mu_a|\LM^a(\Phi(t))-\LM^a(\Phi^*)|
\end{align}

Suppose that two modalities achieve the same optimal value, i.e., $\LM^a(\Phi^*)=\LM^v(\Phi^*)$. Then we have:
\begin{align}
\|\nabla\LM^a(\Phi(t))\|^2
&\ge2\mu|\LM^a(\Phi(t))-\LM^a(\Phi^*)|\\
&=2\mu|\LM^a(\Phi(t))-\LM^v(\Phi^*)|\\
&\ge2\mu|\LM^a(\Phi(t))-\LM^v(\Phi(t))|\\
&=2\mu|\GM(\Phi(t))|\\
&\ge\frac{2\mu}{c}|\GM(\Phi(t))|^2
\end{align}

By setting $\kappa=\sqrt{\frac{2\mu}{c}}$, we have:
\begin{align}
\|\nabla \LM^a(\Phi(t))\|\ge\kappa|\GM(\Phi(t))|
\end{align}

\end{proof}

We have the following theorem:
\begin{thm}[Convergence of $\GM$ with Gradient Boosting]\label{thm:cg-rate}
Under assumption~(\ref{ass:smooth}) and (\ref{ass:classification}), if the learning rate is set as $\eta_t=\frac{\nu}{L_a\beta^2}$, we have:
\begin{align}
\GM(\Phi(T))\le \frac{\GM(\Phi(0))}{1+dT\GM(\Phi(0))}=\OM(\frac{1}{T}),
\end{align}
where $d=\frac{\nu^2\kappa^2}{2L_a\beta^2}$.
\end{thm}

\begin{proof}[Proof of Theorem~\ref{thm:cg-rate}]
Based on smooth assumption, we have:
\begin{align}
\LM^a(\Phi(t+1))
&\le\LM^a(\Phi(t))+\langle\nabla\LM^a(\Phi(t)),\Phi(t+1)-\Phi(t)\rangle+\frac{L_a}{2}\|\Phi(t+1)-\Phi(t)\|^2,\\
&=\LM^a(\Phi(t))+\eta_t\langle\nabla\LM^a(\Phi(t)),h(\Phi(t))\rangle+\frac{L_a\eta_t^2}{2}\|h(\Phi(t))\|^2,\\
&\overset{(\ref{eq2})}{\le}\LM^a(\Phi(t))-\nu\eta_t\|\nabla\LM^a(\Phi(t))\|^2+\frac{L_a\eta_t^2}{2}\|h(\Phi(t))\|^2,\\
&\overset{(\ref{eq3})}{\le}\LM^a(\Phi(t))-\nu\eta_t\|\nabla\LM^a(\Phi(t))\|^2+\frac{L_a\eta_t^2\beta^2}{2}\|\nabla\LM^a(\Phi(t))\|^2
\end{align}

Considering that the parameters for each modality are independent, we have: $\LM^v(\Phi(t+1))=\LM^v(\Phi(t))$. Then we have:
\begin{align}
\GM(\Phi({t+1}))\le\GM(\Phi(t))-(\nu\eta_t-\frac{L_a\eta_t^2\beta^2}{2})\|\nabla\LM^a(\Phi(t))\|^2
\end{align}

By setting $\eta_t^a=\frac{\nu}{L_a\beta^2}$, we have:
\begin{align}
\GM(\Phi({t+1}))\le\GM(\Phi(t))-\frac{\nu^2}{2L_a\beta^2}\|\nabla\LM^a(\Phi(t))\|^2
\end{align}

By using lemma~(\ref{lemma:gradient-norm}), we have:
\begin{align}
\GM(\Phi(t+1))\le\GM(\Phi(t))-\frac{\nu^2\kappa^2}{2L_a\beta^2}|\GM(\Phi(t))|^2.
\end{align}

Then we set $d=\frac{\nu^2\kappa^2}{2L_a\beta^2}$:
\begin{align}
\frac{1}{\GM(\Phi(t+1))}
&\ge\frac{1}{\GM(\Phi(t))-d|\GM(\Phi(t))|^2}\\
&\ge\frac{1}{\GM(\Phi(t))}(1+d\GM(\Phi(t)))\\
&=\frac{1}{\GM(\Phi(t))}+d
\end{align}

By summing up for $k=1,\cdots,T$, we have:
\begin{align}
\GM(\Phi(T))\le\frac{\GM(\Phi(0))}{1+dT\GM(\Phi(0))}=\OM(\frac{1}{T}).
\end{align}
\end{proof}

\section{Notation Definition}
We summarize the notation definition we used in this paper in Table~\ref{tab:notation}.

\begin{table}[!htbp]
\centering
\caption{Notation Definition}
\label{tab:notation}
\begin{tabular}{c|c}
\toprule
Notation                         & Description                                       \\\midrule
$N$                              & The number of training data.                      \\
$K$                              & The number of category labels.                    \\
$\x^a/\x^v$                      & Audio/video data point.                           \\
$\X=\{(\x_i^a,\x_i^v)\}_{i=1}^N$ & Training set.                                     \\
$\y_i\in\{0,1\}^K$               & Category label of $i$-th data.                    \\
$\Y=\{\y_i\}_{i=1}^N$            & Category label set.                               \\
Superscript $o$                  & $o$-modality.                                     \\
Superscript $a/v$                & Audio/video modality.                             \\
$\phi^o(\cdot)$                  & Encoder of $o$ modality.                          \\
$\theta^o$                       & Parameters of $\phi^o(\cdot)$.                    \\
$\u^o$                           & Feature of $\x^o$.                                \\
$\psi^o(\cdot)$                  & Classifier of $o$-modality.                       \\
$\Theta^o$                       & Parameters of $\psi^o(\cdot)$.                    \\
$\p^o$                           & Prediction of $\x^o$.                             \\
$\Phi^o=\{\theta^o,\Theta^o\}$   & Parameters set.                                   \\
$\ell(\cdot)$                    & Cross entropy loss.                               \\
$\LM_{\mathrm{CE}}(\cdot)$       & Cross entropy loss over training.                 \\
$\lambda$                        & Hyper-parameter used to soften label.             \\
$\hat\y^o_{it}$                  & Residual label of $\x^o_i$ at $t$-th iteration.   \\
$\epsilon(\cdot)$                & Classification loss.                              \\
$\epsilon_{\mathrm{all}}(\cdot)$ & Loss for $t$ classifiers.                         \\
$\epsilon_{\mathrm{pre}}(\cdot)$ & Loss for $t-1$ classifiers.                       \\
$L(\cdot)$                       & Objective function.                               \\
$\LM_{\mathrm{SUB}}(\cdot)$      & Overall objective function.                       \\
$n^o$                            & The number of classifier for $o$-modality.        \\
$s_t^o$                          & Confidence for $o$-modality at $t$-th iteration.  \\
$\sigma$                         & Coefficient of confidence score.                  \\
$\tau$                           & Tolerance score.                                  \\
$t_N$                            & Adjustment frequency.                             \\
$\GM(\cdot)$                     & Gap function.                                     \\
\bottomrule
\end{tabular}
\end{table}

\section{Additional Experiments}
\subsection{Datasets}

\noindent{\bf\CREMAD:} \CREMAD~\cite{CREMAD:journals/taffco/CaoCKGNV14} is an audio-visual dataset designed for speech emotion recognition, It comprises 7,442 video clips of 2$\sim$3 seconds from 91 actors speaking several short words. which are divided into 6,698  training samples and 744 testing samples. This dataset includes the six most common emotions: angry, happy, sad, neutral, discarding, disgust and fear.

\noindent{\bf\Kinetics:} \Kinetics~\cite{Kinetics-Sound:conf/iccv/ArandjelovicZ17} dataset is a commonly used dataset containing 31 action categories that can be recognized visually and auditorily, which contains 19,000 10$\sim$second video clips from YouTube. This dataset is divided into training set with 15,000 clips, validation set with 1,900 clips, and testing set with 1,900 clips.

\noindent{\bf\NVGesture:} \NVGesture~\cite{NVGeasture:conf/cvpr/MolchanovYGKTK16} dataset is a multimodal dataset specifically designed for gesture recognition, containing three types of data modalities, i.e., RGB, optical flow~(OF), and Depth. This dataset includes 25 different gesture categories, covering a variety of common gestures. This dataset is divided into 1,050 samples for training and 482 samples for testing.

\noindent{\bf\VGGSound:} \VGGSound~\cite{VGGSound:conf/icassp/ChenXVZ20} dataset is an audio-visual dataset in the wild, with nearly 200K 10-second video
clips. Each sound-emitting object is also visible in the corresponding video clip in this dataset. After filtering out unavailable videos, 168,618 videos for training and validation, and 13,954 videos for testing in experimental settings.

\noindent{\bf\Twitter:} \Twitter~\cite{Twitter15:conf/ijcai/Yu019} dataset is a dataset designed for multilingual sentiment analysis, which is divided into training set with 3,197 pairs, validation set with 1,122 pairs and testing set with 1,037 pairs. This dataset consists of tweets collected from Twitter\footnote{\url{http://twitter.com}} with three different sentiment labels: positive, negative, and neutral.

\noindent{\bf\Sarcasm:} \Sarcasm~\cite{Sarcasm:conf/acl/CaiCW19} dataset is specifically designed for sarcasm detection tasks, containing a large amount of text data labeled as sarcastic or non-sarcastic, which typically draws from various sources, such as social media, news comments, and conversation data, with each entry clearly labeled for sarcasm. This dataset includes 19,816 pairs for the training set, 2,410 pairs for the validation set, and 2,409 pairs for the testing set.

\subsection{Implementations of Toy Experiments}
The toy experiment is designed to validate the motivation behind our method. Specifically, it is conducted on the \CREMAD~dataset, where the task is multimodal classification. Four methods are adopted for comparison: Naive MML, G-Blend~\cite{OGR-GB:WangTF20}, MML w/ GB, and Ours. Here, MML w/ GB refers to applying the gradient boosting algorithm to further improve the trained video model obtained from Naive MML, while keeping the audio model fixed. Ours refers to the method employing the sustained boosting algorithm proposed in this paper. Furthermore, ResNet-18 is used as the encoder for all methods, with its parameters randomly initialized. The selection of other hyperparameters is kept consistent with the main experiment.

\subsection{More Results for Ablation Study}
{\bf MAP Results on \CREMAD~and \Kinetics~Datasets: }We have added the MAP results of the ablation study on the \CREMAD~and \Kinetics~datasets in Table~\ref{tab:added-ablation}. The MAP results in Table~\ref{tab:added-ablation} demonstrate that the objectives defined in Equation~(\ref{eq:residual}),~(\ref{eq:overall}), and~(\ref{eq:previous}) all lead to performance gains.

\begin{table}[!htbp]
\centering
\caption{The MAP for ablation study on \CREMAD~and \Kinetics~datasets.}
\label{tab:added-ablation}
\centering
\begin{tabular}{ccc|ccc|ccc}
\toprule
\multirow{2}{*}{$\epsilon$}&\multirow{2}{*}{$\epsilon_{o}$}&\multirow{2}{*}{$\epsilon_{p}$}& &\CREMAD & & &\Kinetics  & \\
\cmidrule(lr){4-6}\cmidrule(lr){7-9}
            &             &          & Multi   & Audio   & Video   & Multi   & Audio   & Video   \\\midrule
\tikzcmark  &\tikzxmark   &\tikzcmark&0.8967   &0.7061   &0.7523   &0.7796   &0.5465   &0.5484   \\   
\tikzxmark  &\tikzcmark   &\tikzxmark&0.8895   &0.7235   &0.7535   &0.7864   &\bf0.5482&0.5606   \\   
\tikzxmark  &\tikzcmark   &\tikzcmark&0.9014   &0.7321   &0.7335   &0.7821   &0.5474   &0.5651   \\   
\tikzcmark  &\tikzcmark   &\tikzcmark&\bf0.9103&\bf0.7529&\bf0.7612&\bf0.7901&0.5461   &\bf0.6040\\   
\bottomrule
\end{tabular}
\end{table}

{\bf Ablation Study on Image-Text and Tri-Modal Datasets: }To further validate our method, we report an ablation study on image-text dataset, i.e., \Twitter~dataset, and the dataset with three modalities, i.e., \NVGesture~dataset. The accuracy and F1 are reported in Table~\ref{tab:Twitter-ablation} and Table~\ref{tab:NVGesture-ablation}. The results demonstrate that the objectives defined in Equation~(\ref{eq:residual}),~(\ref{eq:overall}), and~(\ref{eq:previous}) all lead to performance gains in terms of both accuracy and F1. 

\begin{table}[!htbp]
\centering
\caption{The results for ablation study on \Twitter~dataset.}
\label{tab:Twitter-ablation}
\centering
\begin{tabular}{ccc|ccc|ccc}
\toprule
\multirow{2}{*}{$\epsilon$}&\multirow{2}{*}{$\epsilon_{o}$}&\multirow{2}{*}{$\epsilon_{p}$}&\multicolumn{3}{c|}{Accuracy}&\multicolumn{3}{c}{F1}\\
\cmidrule(lr){4-6}\cmidrule(lr){7-9}
            &             &          & Multi & Image  & Text  & Multi & Image  & Text    \\\midrule
\tikzcmark  &\tikzxmark   &\tikzcmark&  0.7425 &  0.5400 &  0.7406 & 0.6830 & 0.3957 & 0.6868\\   
\tikzxmark  &\tikzcmark   &\tikzxmark&  0.7396 &  0.5429 &  0.7387 & 0.6787 &\bf0.4113 & 0.6806\\   
\tikzxmark  &\tikzcmark   &\tikzcmark&  0.7445 &  0.5265 &  0.7454 & 0.6767 & 0.3865 & 0.6794\\   
\tikzcmark  &\tikzcmark   &\tikzcmark&\bf0.7512 &\bf0.5434 &\bf0.7488 &\bf0.6962 & 0.3998 &\bf0.6911\\   
\bottomrule
\end{tabular}
\end{table}

\begin{table}[!htbp]
\centering
\caption{The results for ablation study on \NVGesture~dataset.}
\label{tab:NVGesture-ablation}
\centering
\begin{tabular}{ccc|cccc|cccc}
\toprule
\multirow{2}{*}{$\epsilon$}&\multirow{2}{*}{$\epsilon_{o}$}&\multirow{2}{*}{$\epsilon_{p}$}&\multicolumn{4}{c|}{Accuracy}&\multicolumn{4}{c}{F1}\\
\cmidrule(lr){4-7}\cmidrule(lr){8-11}
            &             &          & Multi & RGB  & OF  & Depth  & Multi & RGB  & OF  & Depth  \\\midrule
\tikzcmark  &\tikzxmark   &\tikzcmark&0.8465 & 0.7573 & 0.7780 & 0.7946        & 0.8505 & 0.7649 & 0.7809 & 0.7968\\   
\tikzxmark  &\tikzcmark   &\tikzxmark&0.8402 & 0.7427 & 0.7718 & 0.7759        & 0.8428 & 0.7604 & 0.7720 & 0.7812\\   
\tikzxmark  &\tikzcmark   &\tikzcmark&0.8485 &\bf 0.8008 & 0.7988 & 0.7801     & 0.8526 & \bf 0.8085 & 0.8009 & 0.7881\\   
\tikzcmark  &\tikzcmark   &\tikzcmark&\bf0.8501 & 0.7988 &\bf 0.8029 & \bf 0.7967&\bf 0.8533 & 0.8076 &\bf 0.8057 &\bf 0.8033\\   
\bottomrule
\end{tabular}
\end{table}

\subsection{Evaluation of Generalization in More Scenarios}
To further validate the effectiveness of our method, we employ a actor-independent data partitioning strategy on the \CREMAD~dataset, ensuring that the training, validation, and test sets contain mutually exclusive subjects following the setting of~\cite{Emotion:conf/icassp/GoncalvesB22,GameTheory:conf/nips/Konstantinos25}. Specifically, we randomly select the audio-video data of 63 actors for training, 14 actors for validation, and 14 actors for testing, with no overlap of actors among the three splits.

We first investigate the influence of the hyperparameters $\sigma$ and $\lambda$ on actor-independent \CREMAD~dataset. The results are shown in Figure~\ref{fig:sigma-cremad} and Figure~\ref{fig:lambda-cremad}. The results demonstrate that our method is not sensitive to hyper-parameters $\sigma$ and $\lambda$. Furthermore, the ablation study on actor-independent \CREMAD~dataset is shown in Table~\ref{tab:re-split-cremad-ablation}. The results further verify the effectiveness of the objectives defined in Equation~(\ref{eq:residual}),~(\ref{eq:overall}), and~(\ref{eq:previous}).

\begin{figure}[!htbp]
\centering
\begin{minipage}[t]{0.3\textwidth}
\centering
\includegraphics[width=\linewidth]{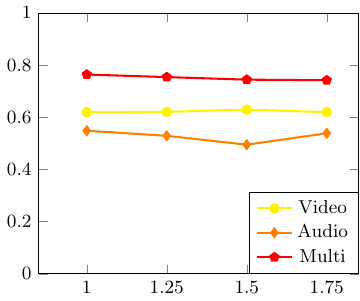}
\caption{Sensitivity to $\sigma$.}
\label{fig:sigma-cremad}
\end{minipage}
\begin{minipage}[t]{0.3\textwidth}
\centering
\includegraphics[width=\linewidth]{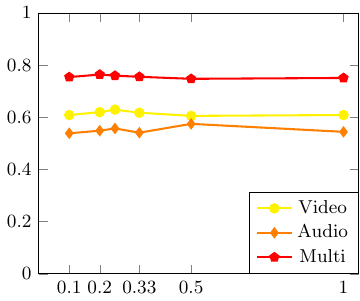}
\caption{Sensitivity to $\lambda$.}
\label{fig:lambda-cremad}
\end{minipage}
\end{figure}

\begin{table}[!htbp]
\centering
\caption{The results for ablation study on actor-independent \CREMAD~dataset.}
\label{tab:re-split-cremad-ablation}
\centering
\begin{tabular}{ccc|ccc|ccc}
\toprule
\multirow{2}{*}{$\epsilon$}&\multirow{2}{*}{$\epsilon_{o}$}&\multirow{2}{*}{$\epsilon_{p}$}&\multicolumn{3}{c|}{Accuracy}&\multicolumn{3}{c}{MAP}\\
\cmidrule(lr){4-6}\cmidrule(lr){7-9}
            &             &          & Multi & Image  & Text  & Multi & Image  & Text    \\\midrule
\tikzcmark  &\tikzxmark   &\tikzcmark& 0.7443 & 0.6197 & 0.5191 & 0.8086	&0.6517	&0.5433 \\   
\tikzxmark  &\tikzcmark   &\tikzxmark& 0.7333 & 0.6262 & 0.5344 & 0.8135	&0.6508	&\bf0.5861 \\      
\tikzxmark  &\tikzcmark   &\tikzcmark& 0.7388 & \bf0.6273 & 0.5366 & 0.8111	&0.6496 &0.5860 \\      
\tikzcmark  &\tikzcmark   &\tikzcmark& \bf0.7639 & 0.6087 & \bf0.5388 & \bf0.8300	&\bf0.6557	&0.5544 \\      
\bottomrule
\end{tabular}
\end{table}

Then, we compare our method with some representative methods on actor-independent \CREMAD~dataset. The results are shown in Table~\ref{tab:cmp-re-split-cremad}. From the results in Table~\ref{tab:cmp-re-split-cremad}, we can find that our method can achieve the best performance compared with SOTA baselines.

\begin{table}[!htbp]
\centering
\caption{Comparison with SOTA baselines on actor-independent \CREMAD~dataset.}
\label{tab:cmp-re-split-cremad}
\begin{tabular}{c|cc}
\toprule
Method     &  Accuracy & MAP  \\\midrule
Audio/Video   & 0.6219/0.5366 & 0.6588/0.5380 \\
Naive MML  & 0.6153 & 0.6671 \\
MLA        & 0.7421 & 0.8100 \\
ReconBoost & 0.7344 & 0.8027 \\
Ours       & \bf0.7639 & \bf0.8300 \\
\bottomrule
\end{tabular}
\end{table}

\subsection{Influence of Unstable Confidence Score}
Since our method relies on confidence gap to determine whether to add a classifier, fluctuations in confidence scores are indeed an influential factor. We conduct an experiment to explore the fluctuation of confidence score in scenario with modality noise on \CREMAD~dataset. Specifically, we add Gaussian noise to the video modality to introduce instability in confidence scores and investigate its impact on the algorithm. We design a metric to evaluate the fluctuation of confidence score during training. We first compute the absolute difference in confidence scores between two consecutive iterations, and then take the average over all training rounds:
\begin{align}
\forall o\in\{a,v\},\;\bar s^o=\frac{1}{n}\sum_{t=2}^n\vert s^o_t-s^o_{t-1}\vert.
\end{align}
Here, $s^o_t$ denotes the confidence score of $o$-modality at $t$-th iteration, and $n$ denotes the current iteration index in the training process. A higher value of $\bar s^o$ indicates more severe fluctuations in confidence scores, reflecting a sharper and less stable prediction behavior. Furthermore, we compare our method with naive MML and ReconBoost~\cite{ReconBoost:conf/icml/CongHua24}.

We first compare the model performance in scenario with modality noise in Table~\ref{tab:confidence-score}, where $\zeta$ denotes the noise rate, ``\#CLS'' denotes the number of classifier during training, and in the first column, we also report the fluctuation of confidence score. In fact, as training progressed, we observe that in noisy scenarios, increasing the number of classifiers beyond a certain point no longer led to improved performance on the validation set. This phenomenon became more pronounced as the noise level increased. In other words, the classifiers may overcompensate in scenarios involving modality noise. On the other hand, we also observe that after reaching its optimal performance, the model's performance gradually degraded without a drastic drop. Therefore, in practical scenarios, an early stopping strategy based on validation set performance can be employed to select the optimal model and avoid classifier overcompensation.

\begin{table}[!htbp]
\centering
\caption{The fluctuation of confidence score.}  
\label{tab:confidence-score}
\begin{tabular}{r|cccc}
\toprule
\#CLS   & $\zeta=0,(\bar s^v=0.1428)$& $\zeta=20,(\bar s^v=0.3920)$& $\zeta=50,(\bar s^v=0.4760)$\\\midrule
1       &  0.1411  	                 &  0.1411	                   &  0.2863                     \\
2       &  0.6559                    &  0.6129                     &  0.6398                     \\
3       &  0.7634                    &  0.7608                     &  0.7728                     \\
4       &  0.8038                    &  0.7917                     &  0.8185                     \\
5       &  0.8266                    &  0.8118                     &  0.8212                     \\
6       &  0.8306                    &  0.8293                     &\bf0.8253                    \\
7       &  0.8401                    &\bf0.8333	                   &  0.8199                     \\
8       &  0.8441                    &  0.8266                     &  0.8239                     \\
9       &  0.8468                    &  0.8293                     &  0.8185                     \\
10      &\bf0.8515                   &  0.8280                     &  0.8145                     \\
\bottomrule
\end{tabular}
\end{table}

Furthermore, we compare the performance with competitive baselines including naive MML, ReconBoost~\cite{ReconBoost:conf/icml/CongHua24}. The results are shown in Table~\ref{tab:cmp-noise}. We can find that compared with the competitive method ReconBoost, our approach consistently achieves superior performance, demonstrating its effectiveness in scenario with modality noise.

\begin{table}[!htbp]
\centering
\caption{Performance comparison under modality noise scenario.}  
\label{tab:cmp-noise}
\begin{tabular}{r|ccc}
\toprule
Method      & $\zeta=0$   & $\zeta=20$   & $\zeta=50$   \\\midrule
Naive MML   & 0.6507      & 0.6425       & 0.6331       \\\midrule
ReconBoost  & 0.7557	  & 0.7215       & 0.7031       \\
Ours        &\bf 0.8515   & \bf 0.8333   & \bf 0.8253   \\
\bottomrule
\end{tabular}
\end{table}

\subsection{Computational and Memory Cost at Inference Stage}
To demonstrate the practical applicability of our approach, we further provide a comprehensive analysis of its computational and memory overhead at inference stage.\footnote{For computational cost during training, the analysis have been posted in Section~\ref{sec:further-analysis} of the original paper.} Furthermore, we conduct an experiment to analyze the computational and memory cost during inference phase on \CREMAD~dataset. The baselines include naive MML, PMR~\cite{PMR:conf/cvpr/Fan0WW023}, AGM~\cite{AGM:conf/iccv/LiLHLLZ23}, MLA~\cite{MLA:conf/cvpr/ZhangYBY24}, ReconBoost~\cite{ReconBoost:conf/icml/CongHua24}. The results are shown in Table~\ref{tab:cmp-inference-time}. We can find that our method achieves superior performance while maintaining competitive inference time. Furthermore, compared with naive MML, our method introduces only 1M additional parameters when 10 classifiers are added. Given the total model size of 23.6M, this corresponds to a relatively small increase of approximately 4\%.

\begin{table}[!htbp]
\centering
\caption{Computational cost comparison on \CREMAD~dataset.}  
\label{tab:cmp-inference-time}
\begin{tabular}{r|cc}
\toprule
Method    &  Accuracy   & Inference time~(s)  \\\midrule
Naive MML &  0.6507     & 5.29                \\
PMR       &  0.6659     & 6.59                \\
AGM       &  0.6733     & 5.58                \\
MLA       &  0.7943     & 5.63                \\
ReconBoost&  0.7557     & 5.33                \\
Ours      &  0.8515     & 5.62                \\
\bottomrule
\end{tabular}
\end{table}

\subsection{Selection of Score Function}
In our method, the score function is employed to quantify the disparity in classification performance between different modalities, which serves as the basis for determining whether to introduce an additional classifier for the weak modality. Therefore, any metric that effectively captures the difference in classification capabilities across modality-specific models can be adopted as a score function.

To some extent, entropy and loss can reflect the learning status of modality-specific models and thereby serve as proxies for their classification capabilities. To evaluate the effectiveness of using entropy and loss as score functions, we conducted experiments on the \CREMAD~dataset, where these two metrics were used to guide the adaptive classifier assignment. The experimental results are presented in Table~\ref{tab:score-selection}, where $\tau$ denotes the threshold. Please note that since we use the ratio between different modal metrics to compare against the threshold $\tau$, the value of $\tau$ remains consistent across different metric types such as confidence, loss, and entropy. These findings indicate that the proposed method is robust to the choice of score function, consistently achieving comparable results regardless of whether confidence score, entropy, or loss is used.

\begin{table}[!htbp]
\centering
\caption{Performance with different score function.}  
\label{tab:score-selection}
\begin{tabular}{r|ccc}
\toprule
Score function          &   Accuracy   &  Threshold $\tau$  & \#CLS   \\\midrule
Confidence score~(Ours)	&   0.8515     &  0.01              & 10      \\
Entropy                 &   0.8522     &  0.01              & 10      \\     
Loss                    &   0.8562     &  0.01              & 10      \\    
\bottomrule
\end{tabular}
\end{table}

\subsection{Computational Cost vs. The Number of Classifiers}
The computational overhead—including both training and inference time—is indeed influenced by the number of classifiers. Since our method dynamically adds classifiers during training, we incorporate a threshold to limit the maximum number of classifiers. This allows us to systematically investigate the trade-off between computational cost and model performance as a function of classifier quantity.

Specifically, we conduct experiments on the \CREMAD~dataset, where the maximum number of weak modality classifiers is constrained to not exceed a predefined threshold $M$. We vary $M$ across the set $\{2, 4, 6, 8, 10\}$, and for each setting, we report the training time, inference time, and the corresponding classification performance.

The experimental results are summarized in Table~\ref{tab:num-classifier}. We observe that as the number of classifiers increases, both training time and inference time increase accordingly, while performance also improves to a certain extent. This indicates that incorporating more classifiers can indeed enhance model performance, but at the cost of increased computational overhead.

\begin{table}[!htbp]
\centering
\caption{Computational cost vs the number of classifiers.}  
\label{tab:num-classifier}
\begin{tabular}{r|ccc}
\toprule
\#CLS    & Accuracy   & Training time~(hrs)  & Inference time~(s)  \\\midrule
2        & 0.8159     & 1.68                 & 5.31                \\
4        & 0.8387     & 1.72                 & 5.38                \\
6        & 0.8441     & 1.78                 & 5.46                \\
8        & 0.8468     & 1.86                 & 5.53                \\
10       & 0.8515     & 1.98                 & 5.62                \\
\bottomrule
\end{tabular}
\end{table}

\end{document}

%% file: defs.tex
\def\p{{\boldsymbol p}}

\def\u{{\boldsymbol u}}

\def\X{{\boldsymbol X}}
\def\x{{\boldsymbol x}}
\def\Y{{\boldsymbol Y}}
\def\y{{\boldsymbol y}}

\def\GM{{\mathcal G}}

\def\LM{{\mathcal L}}

\def\OM{{\mathcal O}}

\def\mod{\mathtt{mod}}